\documentclass{article}
\usepackage{journal}
\jmlrheading{}{2026}{}{}{}{O. Montasser}
\ShortHeadings{Montasser}{Sample and Oracle Complexity of Robust Learning}
\usepackage{amsmath,amssymb,amsthm,mathtools}
\usepackage{charter}
\usepackage{dsfont}
\usepackage{thmtools}
\usepackage{algorithm}
\usepackage[algo2e,ruled,linesnumbered,vlined]{algorithm2e}
\usepackage[noend]{algpseudocode}
\usepackage{microtype}
\usepackage[dvipsnames]{xcolor}
\usepackage{hyperref}
\usepackage{prettyref}
\usepackage[capitalise,nameinlink]{cleveref}
\usepackage{nicefrac}
\usepackage{enumitem}
\hypersetup{colorlinks=true,citecolor=Maroon,linkcolor=Maroon,urlcolor=Maroon}
\usepackage{comment}
\usepackage{mathrsfs}
\usepackage{tikz}
\usetikzlibrary{arrows.meta}

\newrefformat{alg}{\Cref{#1}}
\newrefformat{thm}{\Cref{#1}}
\newrefformat{corr}{\Cref{#1}}
\newrefformat{lem}{\Cref{#1}}
\newrefformat{app}{Appendix~\ref{#1}}

\newtheorem{lemma}{Lemma}
\newtheorem{corollary}{Corollary}
\theoremstyle{definition}
\newtheorem{remark}{Remark}

\newcommand{\X}{\mathcal X}
\newcommand{\Y}{\mathcal Y}
\newcommand{\Cc}{\mathcal C}
\newcommand{\Fc}{\mathcal F}
\newcommand{\Uc}{\mathcal U}
\newcommand{\E}{\mathds E}
\newcommand{\B}{\mathrm B}
\newcommand{\Risk}{{\mathrm R}_\Uc}
\newcommand{\Prob}{\mathds P}
\newcommand{\1}{\mathds 1}
\newcommand{\vc}{\mathrm{VCdim}}
\newcommand{\eps}{\varepsilon}
\newcommand{\prn}[1]{\left(#1\right)}
\newcommand{\brk}[1]{\left[#1\right]}
\newcommand{\set}[1]{\left\{#1\right\}}
\newcommand{\ceil}[1]{\left\lceil#1\right\rceil}
\newcommand{\Ex}[1]{\underset{#1}{\E}}
\newcommand{\ERM}{\mathrm{ERM}}
\newcommand{\RERM}{\mathrm{RERM}}
\newcommand{\MAJ}{\mathrm{MAJ}}

\newcommand{\Law}{\operatorname{Law}}
\newcommand{\Unif}{\operatorname{Unif}}
\newcommand{\KL}{\operatorname{KL}}
\newcommand{\TV}{\operatorname{TV}}
\newcommand{\Alg}{\mathop{\mathsf{ALG}}\nolimits}

\title{Bagging Robustly Learns VC Classes \\with Linear Sample Complexity}
\author{\name{Omar Montasser} \email{omar.montasser@yale.edu}\\
    \addr Yale University\\}

\begin{document}

\maketitle

\begin{abstract}
    We revisit the problem of learning predictors robust to adversarial examples at test-time. We prove that VC classes are adversarially robustly learnable with sample complexity {\em linear} in the VC dimension $d$, providing an {\em exponential} improvement over the previous upper bound of \citet*{DBLP:conf/colt/MontasserHS19}. Remarkably, this result is achieved with a simple improper algorithm that combines the classic heuristic {\em bagging} (bootstrap aggregation) of \citet*{DBLP:journals/ml/Breiman96b} with {\em robust empirical risk minimization} (RERM). Our algorithm computes RERMs on $O(d^\star)$ indpendent bootstrap samples and outputs their majority-vote, where $d^\star$ denotes the dual VC dimension. We complement this result with a lower bound showing that this is unavoidable: in general, any learner in this oracle model requires $\Omega(d^\star)$ calls to an $\RERM$ oracle, even when given arbitrarily many training examples.
\end{abstract}
    
\section{Introduction}

Learning predictors robust to adversarial examples is a major contemporary challenge in machine learning. Adversarial examples can be thought of as carefully crafted perturbations of test examples that cause predictors to misclassify. Over the past decade, there has been a significant interest in how deep learning predictors are {\em not} robust to adversarial examples \citep[][]{szegedy2013intriguing,biggio2013evasion,DBLP:journals/corr/GoodfellowSS14}, leading to an ongoing effort to devise methods for learning predictors that are adversarially robust \citep[e.g.,][]{DBLP:conf/iclr/MadryMSTV18, DBLP:conf/icml/CohenRK19, DBLP:conf/icml/ZhangYJXGJ19}.

Given an instance space $\X$ and label space $\Y=\set{-1,1}$, we formalize an adversary (or perturbation map) we would like to be robust against as $\Uc: \X \to 2^{\X}$, where $\Uc(x)\subseteq \X$ is the set of perturbations of $x$ that can be chosen by the adversary at test-time. For example,~$\mathcal{U}$ could be perturbations of bounded $\ell_{\infty}$ norm which represents ``imperceptible'' image perturbations in practice \citep{DBLP:journals/corr/GoodfellowSS14}. The only (implicit) restriction is that $\Uc(x)$ is non-empty for every $x$. We observe $n$ i.i.d.~samples $S=((X_i, Y_i))_{i=1}^n\sim P^n$ from an (unknown) distribution $P$ over $\X \times \Y$, and our goal is to learn a predictor $\hat{f}:\X\to \Y$ having small population {\em robust} risk,

\begin{equation}
    \label{eqn:robrisk}
    \Risk(\hat{f}; P) = \Ex{(X,Y)\sim P} \brk{ \sup_{Z\in \Uc(X)} \1\set{\hat{f}(Z)\neq Y}}.
\end{equation}

A central approach to adversarially robust learning in practice is {\em adversarial training}, where model parameters are trained to fit the training examples and their corresponding adversarial perturbations \citep[e.g.,][]{DBLP:conf/iclr/MadryMSTV18, DBLP:conf/icml/ZhangYJXGJ19}. Theoretically, this can be viewed as minimizing the empirical robust risk over a function class $\Fc \subseteq \Y^\X$ (e.g., neural networks):
\begin{equation}
    \label{eqn:rerm}
    \widehat{f}_S \in \RERM_\Fc(S):= \underset{f\in \Fc}{\arg\min}\quad \frac{1}{n} \sum_{i=1}^{n} \sup_{Z_i\in \Uc(X_i)} \1\set{f(Z_i) \neq Y_i}.
\end{equation}
This approach is justified by classical uniform convergence arguments applied to the robust risk, i.e.,  arguing that with sufficiently many training examples, the robust generalization gap $|\Risk(f;P)-\Risk(f;S)|$ is small for all $f\in \Fc$ \citep*{DBLP:conf/nips/CullinaBM18,bubeck2019adversarial,DBLP:conf/icml/YinRB19}. However, empirical evidence shows \textit{severe overfitting}: adversarial training yields models with low empirical robust risk but high population robust risk \citep*{schmidt2018adversarially}. This discrepancy suggests that adversarial training and uniform convergence alone may not guarantee provable robust learning.

This empirical observation is indeed mirrored by a fundamental theoretical limitation. \citet*{DBLP:conf/colt/MontasserHS19} showed that minimizing the empirical robust risk as in \eqref{eqn:rerm}--or any surrogate thereof--can {\em provably fail}, even in simple settings. Specifically, there are function classes $\Fc$ with VC dimension $1$ that are {\em not} adversarially robustly learnable with any {\em proper} learning algorithm, namely, one constrained to output a predictor $\hat{f}\in\Fc$. To bypass the limitation of proper learning, \citet*{DBLP:conf/colt/MontasserHS19} designed an {\em improper} learning algorithm to adversarially robustly learn classes $\Fc$ with finite VC dimension. This separation between proper and improper learning stands in sharp contrast to classical (non-robust) PAC learning \citep{DBLP:journals/cacm/Valiant84}, where empirical risk minimization ($\ERM$), a proper learning rule, learns VC classes with near-optimal sample complexity \citep*{VapnikChervonenkis71,DBLP:journals/cacm/Valiant84, DBLP:journals/jacm/BlumerEHW89, DBLP:journals/iandc/EhrenfeuchtHKV89}.

Despite this progress in adversarially robust learning, two fundamental challenges remain concerning: {\em sample-efficiency} and {\em oracle-efficiency}. First, the best known upper bound on the sample complexity is {\em exponential} in the VC dimension $d$ \citep*{DBLP:conf/colt/MontasserHS19}. Second, the improper learning algorithm attaining this bound is quite complex and inefficient; it makes $O(n^{d})$ oracle calls to $\RERM$ \eqref{eqn:rerm} with a prohibitive additional computational overhead of $n^{2^{O(d)}}$. These challenges motivate the central question we study in this work:
\begin{center}
    \textit{Can we design sample-efficient and oracle-efficient adversarially robust learning algorithms?\\Or are there fundamental limitations prohibiting that?}
\end{center}
In addition to our goal of improving the sample complexity of adversarially robust learning, we are especially interested in oracle-efficient algorithms. Just as efficient PAC learning reduces to efficient $\ERM$, this work seeks an analogous recipe for robustness: reducing efficient robust PAC learning of a class $\Fc$ to efficient $\RERM$. In this context, we view $\RERM$ \eqref{eqn:rerm} as a natural black-box oracle, since adversarial training methods in practice \citep[e.g.,][]{DBLP:conf/iclr/MadryMSTV18, DBLP:conf/icml/ZhangYJXGJ19} can be viewed as heuristics implementing $\RERM$. Theoretically, oracle-efficient robust learners would strengthen reductions for leveraging non-robust learners \citep*{DBLP:conf/nips/MontasserHS20}, handling unknown perturbations \citep*{DBLP:conf/colt/MontasserHS21}, and achieving computational efficiency \citep*{DBLP:conf/icml/MontasserGDS20}, thereby unifying algorithmic efficiency with provable robustness.

\subsection{Main Results}

In this work, we prove that VC classes are adversarially robustly learnable with sample complexity {\em linear} in the VC dimension, providing an {\em exponential} improvement over the best previous bound due to \citet*{DBLP:conf/colt/MontasserHS19}. Remarkably, this result is achieved with a simple improper algorithm that combines the classic heuristic {\em bagging} (bootstrap aggregation) due to \citet*{DBLP:journals/ml/Breiman96b} with robust empirical risk minimization ($\RERM$, \eqref{eqn:rerm}). Given a training sample $S$, the algorithm constructs $N$ independent {\em bootstrap samples} $S'_1,\dots, S'_N$ (as specified below), runs $\RERM_\Fc$ on each bootstrap sample $S'_i$ to produce a predictor $\widehat{f}_{S'_i}\in \Fc$, and finally outputs a majority-vote over $\widehat{f}_{S'_1},\dots, \widehat{f}_{S'_N}$. The formal procedure appears in \prettyref{alg:bagging}, followed by its guarantee.

\refstepcounter{algocf}
\label{alg:bagging}
{\setlength{\fboxsep}{4pt}
    \begin{center}
        \fbox{\begin{minipage}{0.96\linewidth}
                {
                    \centering \textbf{Algorithm~\thealgocf: Bagging Robust ERMs}\\
                }
                \textbf{Input:} Training set $S=((X_j,Y_j))_{j=1}^n$, confidence $\delta$, and an $\RERM_\Fc$ oracle $\widehat f$ \eqref{eqn:rerm}.\par
                1. Set $N=O(d^\star+\log(1/\delta))$, and $J_n=\set{n/4,\dots,n-1}$.\par
                2. For each $i=1,\dots,N$:\par
                \hspace*{1.25em}3. Sample $t$ uniformly from $J_n$.\par
                \hspace*{1.25em}4. Sample a bootstrap $S'_i$ of $t$ samples uniformly with replacement from $S_{\leq t}=((X_j, Y_j))_{j=1}^{t}$.\par
                \hspace*{1.25em}5. Run $\RERM_\Fc$ oracle $\widehat{f}$ on $S'_{i}$, denoting its output predictor by $\widehat{f}_{S'_i}$.\par
                \textbf{Output:} The majority-vote predictor $\MAJ\prn{\widehat f_{S'_1},\dots,\widehat f_{S'_N}}$.
        \end{minipage}}
        \smallskip
    \end{center}
}

\begin{restatable}[Realizable]{theorem}{realizablethm}
    \label{thm:linearvc}
    For any function class $\Fc$ with VC dimension $d$ and dual VC dimension $d^\star$, any perturbation set $\Uc$, any deterministic $\RERM_\Fc$ oracle $\widehat{f}$, any distribution $P$ over $\X\times \Y$ where $\inf_{f\in\Fc}\Risk(f; P)=0$, for every $n\geq 4$ and every $\delta \in (0,1)$, letting $N=O(d^\star+\log(1/\delta))$, with probability $1-\delta$ over the random draw of a training dataset $S\sim P^n$ and $N$ bootstrap samples $S'_1\dots, S'_N \subseteq S$,
    \[
        \Risk\prn{ \MAJ(\widehat{f}_{S'_1},\dots, \widehat{f}_{S'_N}); P} = O\prn{\frac{d}{n} + \frac{1}{n}\log\prn{\frac{1}{\delta}}}.
    \]
\end{restatable}

We remark that the oracle complexity (number of calls to $\RERM$) of our algorithm above satisfies $N=O(d^\star)$ where $d^\star$ denotes the dual VC dimension of $\Fc$\footnote{The dual VC dimension, which is the VC dimension of the dual class, is known to satisfy $d^\star \leq 2^{d+1} - 1$ \citep*{Assouad83}, which is tight for some classes but for many natural classes $d^\star = d$ (e.g., halfspaces) or polynomially related $d^\star = {\rm poly}(d)$ (e.g., feed-forward network with threshold activations) \citep*{DBLP:conf/nips/MontasserHS20}}, with a computational overhead of just $O(d^\star n)$. This is a significant improvement over the prior improper algorithm of \citet*{DBLP:conf/colt/MontasserHS19}, which incurred oracle complexity of $O(n^d)$ and additional computational overhead of $O(n^{dd^\star})$. Another significant advantage of our bagging algorithm,  which is of practical importance, is that it can be {\em parallelized} since the $N$ bootstrap samples are independent. The prior improper algorithm of \citet*{DBLP:conf/colt/MontasserHS19} relies on boosting which is inherently sequential and cannot be parallelized in general \citep{DBLP:conf/alt/KarbasiL24}.

By combining the agnostic-to-realizable reduction of \citet*[][Theorem 8]{DBLP:conf/colt/MontasserHS19} with the realizable guarantee of \prettyref{thm:linearvc}, we obtain the following immediate corollary improving the sample and oracle complexity of {\em agnostic} adversarially robust learning.
\begin{corollary}[Agnostic]
    \label{corr:agnostic}
    There is an algorithm $\Alg$ so that for any class $\Fc$ with VC dimension $d$ and dual VC dimension $d^\star$, any perturbation set $\Uc$, any deterministic $\RERM_\Fc$ oracle $\widehat{f}$, the following holds. For any distribution $P$ over $\X\times \Y$, for every $n\geq 4$ and every $\delta \in (0,1)$, with probability $1-\delta$ over $S\sim P^n$ and randomness of $\Alg$, $\Alg$ makes at most $O(d^\star\prn{\log(n)+\log(1/\delta)})$ calls to oracle $\widehat{f}$ and returns a predictor $\widehat{h}_S$ satisfying
    \[
        \Risk\prn{\widehat{h}_S; P} \leq \inf_{f\in \Fc} \Risk(f; P) + O\prn{\sqrt{\frac{d}{n}\log^2(n)+\frac{1}{n}\log\prn{\frac{1}{\delta}}}}.
    \]
\end{corollary}
The agnostic algorithm $\Alg$ works as follows. Denote by $\mathcal{W}$ an instantiation of \prettyref{alg:bagging} with fixed parameters $\eps_0=1/3,\delta_0=1/3$, to be treated as a {\em weak} robust learner. Given a training sample $S$, call $\RERM_\Fc$ once on $S$ to find a maximal subsequence $S'$ that is robustly realizable. Then, run a classical boosting algorithm \citep[e.g.,~$\alpha$-Boost,][Section 6.4.2]{SchapireFreund12} on $S'$ with $\mathcal{W}$ as the weak learner. The final output is majority-of-majority of predictors in $\Fc$.

\begin{remark} [Computational Efficiency]
    Our results provide a recipe for computationally efficient robust PAC learning. In particular, if a concept class $\Fc$ and perturbation set $\Uc$ admit an efficient algorithm implementing $\RERM_\Fc$ \eqref{eqn:rerm}, then \prettyref{thm:linearvc} implies that $\Fc$ is efficiently adversarially robustly PAC learnable with respect to $\Uc$ in the realizable setting. For the agnostic setting, the same implication follows from \prettyref{corr:agnostic} under an additional requirement that the robust loss for $f\in \Fc$ can be evaluated efficiently.
\end{remark}

In light of \prettyref{thm:linearvc}, a natural follow-up question is whether the oracle complexity can be further improved. Because the dual VC dimension satisfies $d^\star \leq 2^{d+1}-1$ with the inequality being tight for some classes \citep*{Assouad83}, we ask whether the oracle complexity can instead be bounded by a polynomial or even a linear function of the VC dimension $d$. Recent advances in classical PAC learning suggest that such an improvement might be possible. \citet{DBLP:conf/colt/Larsen23} showed that bagging with $O(\log n)$ calls to an $\ERM$ oracle yields an optimal PAC learner, and subsequent work reduced this number to just three independent $\ERM$ calls that are combined by a majority-vote \citep*{DBLP:conf/colt/Aden-AliHLZ24,DBLP:journals/corr/abs-2606-13614}. Can an analogous guarantee be obtained for adversarially robust learning? Our second main result gives a negative answer: unlike in classical PAC learning, the oracle complexity of adversarially robust learning must inherently depend on the dual VC dimension $d^\star$. We state this result formally below.

\begin{restatable}{theorem}{oraclelowerboundthm}\label{thm:lowerbound}
    For every integers $d,m\geq1$, there is an instance space
    $\X$, a perturbation map $\Uc$, and a finite collection of classes $\mathscr{F}=\set{\Fc}$ where each class $\Fc$ has VC dimension $d$ and dual VC dimension $d^\star=2^{d+1}-1$ such that the following holds. For any (randomized) learning algorithm $\Alg$ making at most $d^\star-1$ calls to $\RERM$, there exists a class $\Fc \in \mathscr{F}$ and a distribution $P$ over $\X\times \Y$ such that
    \begin{itemize}
        \item $P$ is {\em robustly realizable}, i.e., $\inf_{f\in \Fc} \Risk(f;P)=0$.
        \item With probability at least $1/3$ over $S\sim P^m$ and randomness of $\Alg$, $\Risk(\Alg(S); P) > 1/5$.
    \end{itemize}
\end{restatable}
We remark that this result proves that $\Omega(d^\star)$ oracle queries to $\RERM$ are necessary {\em regardless} of the number of samples $m$ provided to the algorithm. \prettyref{fig:sample-oracle-transition} illustrates the sample-oracle complexity tradeoff and the sharp transition implied by \prettyref{thm:linearvc} and \prettyref{thm:lowerbound}.

\begin{figure}[t]
    \centering
    \begin{tikzpicture}[
        x=1cm,
        y=0.65cm,
        font=\small,
        >=Latex
        ]
        \def\Ncrit{4.2}
        \def\xmax{9.0}
        \def\ncrit{1.25}
        \def\ytop{4.6}
        \def\yaxis{5.0}

        \fill[red!10] (0,0) rectangle (\Ncrit,\ytop);
        \fill[green!12] (\Ncrit,\ncrit) rectangle (\xmax,\ytop);
        \fill[red!10] (\Ncrit,0) rectangle (\xmax,\ncrit);

        \draw[->] (0,0) -- (9.4,0);
        \draw[->] (0,0) -- (0,\yaxis)
        node[above left,align=center] {sample size \(n\)};

        \node[anchor=west] at (9.55,-0.12)
        {RERM-oracle calls \(N\)};

        \draw[very thick,densely dashed]
        (\Ncrit,0) -- (\Ncrit,\ytop);
        \node[above,align=center] at (\Ncrit,\ytop)
        {\textbf{sharp transition}};
        \node[anchor=north] at (\Ncrit,-0.12)
        {\(N\asymp d^\star\)};

        \draw[densely dashed,green!45!black]
        (\Ncrit,\ncrit) -- (\xmax,\ncrit);
        \draw (-0.08,\ncrit) -- (0.08,\ncrit);
        \node[left] at (-0.08,\ncrit)
        {\(n(\eps,\delta)\)};

        \node[
        align=center,
        text width=3.6cm,
        text=red!65!black
        ] at (2.1,2.9) {
            \textbf{Impossible in general}\\[2pt]
            Some VC classes are not robustly
            learnable with \(N\le d^\star-1\) oracle calls.
        };

        \node[
        align=center,
        text width=4.2cm,
        text=green!35!black
        ] at (6.6,2.8) {
            \textbf{All VC classes learnable}\\[2pt]
            \(N=O\!\bigl(d^\star+\log(1/\delta)\bigr)\)\\[1pt]
            \(n(\eps,\delta)=O\!\left(
            \dfrac{d+\log(1/\delta)}{\varepsilon}
            \right)\)
        };

        \node[
        align=center,
        text=red!65!black,
        font=\scriptsize
        ] at (6.6,0.55) {
            \textbf{Insufficient samples in general}
        };
    \end{tikzpicture}
    \caption{A sharp sample--oracle complexity tradeoff: below \(d^\star\)
        oracle calls, robust learnability can fail regardless of sample size \(n\).}
    \label{fig:sample-oracle-transition}
\end{figure}
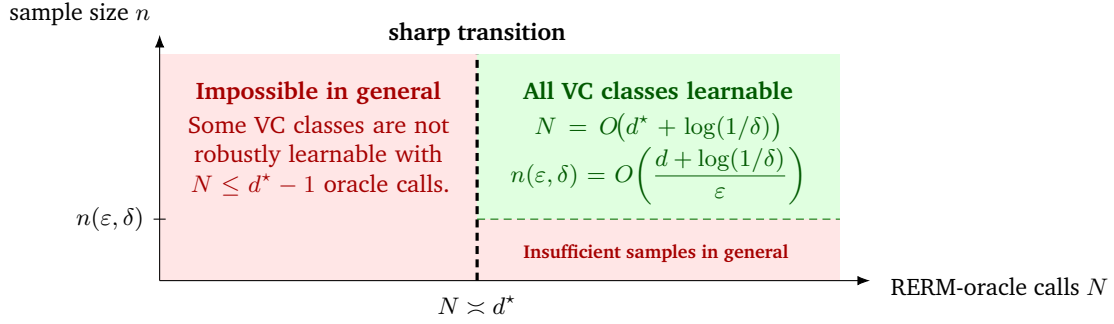

\section{Discussion, Implications, and Related Work}

We discuss connections between our work and related literature. We start by discussing several quantitative improvements on prior works that follow from our results. We then situate our work within broader theoretical models of robustness.

\paragraph{Characterization and Optimal Sample Complexity.} \citet*{DBLP:conf/nips/MontasserHS22} proposed a {\em refined} complexity measure denoted ${\rm dim}(\Fc, \Uc)$ that characterizes which classes $\Fc$ and perturbation sets $\Uc$ are adversarially robustly learnable. This characterization is based on a generalization of the classical one-inclusion graph of \citet*{DBLP:journals/iandc/HausslerLW94}. Our result in \prettyref{thm:linearvc}, and more directly \prettyref{lem:leave-one-out-margin}, implies that this complexity measure is bounded by the VC dimension: ${\rm dim}(\Fc, \Uc) \leq O(d)$, positively resolving Conjecture 3 of \citet*[][]{DBLP:conf/nips/MontasserHS22}.

\paragraph{Reductions in Adversarially Robust Learning.} Several works have explored access to different forms of oracles and studied what robust learning guarantees are possible. This includes access to non-robust learners such as $\ERM$ \citep*{DBLP:conf/nips/MontasserHS20,DBLP:conf/aistats/AhmadiBMS24}, and access to ``attack oracles'' for $\Uc$ \citep*{DBLP:conf/icml/MontasserGDS20, DBLP:conf/colt/MontasserHS21}. A unifying theme across these works is that each specific form of access is leveraged to implement an $\RERM$ oracle \eqref{eqn:rerm}. Consequently, our algorithmic result in \prettyref{thm:linearvc} leads to immediate quantitative improvements when combined with those earlier works in their respective settings. For example, \prettyref{thm:linearvc} implies an improved polynomial attack-oracle complexity of $O({\rm poly}(d, d^\star)\cdot {\rm lit}(\Fc))$ in the \textit{perfect attack oracle} framework of \citet*[][Theorem 2]{DBLP:conf/colt/MontasserHS21} over the previously exponential bound $O({\rm exp}(d, d^\star)\cdot {\rm lit}(\Fc))$, where ${\rm lit}(\Fc)$ denotes the Littlestone dimension of $\Fc$ \citep{DBLP:journals/ml/Littlestone87}.

\paragraph{Adversarially Robust Learning with Tolerance.} Several prior works have studied a relaxed model of adversarially robust learning where the learner competes with a ``larger'' perturbation set $\mathcal{V} \supseteq \Uc$, while the adversary is restricted to perturbations in $\Uc$ \citep*{DBLP:conf/alt/AshtianiPU23, DBLP:conf/alt/BhattacharjeeH023, DBLP:conf/colt/AshtianiPU25}. Concretely, consider metric balls in $p$ dimensions $\Uc(x)={\rm Ball}_r(x)$ and $\mathcal{V}(x)={\rm Ball}_{(1+\alpha)r}(x)$. Previously, the best known upper bound on sample complexity was $\widetilde{O}\prn{{d\prn{\log(d)+p\log(1+1/\alpha)}+\log(1/\delta)}/{\eps}}$ \citep*{DBLP:conf/colt/AshtianiPU25}. \prettyref{thm:linearvc} improves the sample complexity to $O(d/\eps)$ removing dependence on the ambient dimension $p$ and $\alpha$.

\paragraph{Bounded Cardinality Perturbation Sets.} In the special case of adversarially robust learning where $\max_{x\in \X}|\Uc(x)|\leq k$ for some fixed $k \in \mathbb{N}$, \citet*{DBLP:journals/jmlr/AttiasKM22} established a sample complexity upper bound of $\tilde{O}(d\log(k)/\eps)$ via a single $\RERM$ call and uniform convergence analysis. \prettyref{thm:linearvc} improves the sample complexity to $O(d/\eps)$ removing the $\log(k)$ factor entirely, albeit with an improper learning algorithm making $O(d^\star)$ calls to $\RERM$.

\paragraph{Broader Models of Robust Learning.}
Our work studies supervised classification under test-time perturbations.
Adversarial robustness has also been studied in semi-supervised and regression settings \citep*{DBLP:conf/nips/CarmonRSDL19,DBLP:conf/nips/AlayracUHFSK19,DBLP:conf/nips/AttiasHM22,DBLP:conf/icml/AttiasH23}.
A complementary line considers adversarial corruptions at {\em training-time}, beginning with the malicious- and nasty-noise models of PAC learning \citep*{DBLP:journals/siamcomp/KearnsL93,DBLP:journals/tcs/BshoutyEK02}, which were recently shown to be equivalent for efficient distribution-independent learning \citep*{DBLP:conf/soda/BlancHMS26}.
Other theoretical works study clean-label and instance-targeted poisoning and robustly reliable prediction under poisoning attacks \citep*{DBLP:conf/colt/BlumHQS21,DBLP:conf/nips/HannekeKMMM22,DBLP:conf/colt/BalcanBHS22}, as well as learning with monotone adversaries \citep*{LarsenPabbarajuShetty26,Mehrotra26}.
Beyond modeling adversarial examples through prescribed perturbation sets, abstention can provide meaningful guarantees under arbitrary test inputs.
In particular, PQ learning combines labeled samples from a training distribution $P$ with unlabeled samples from an arbitrary test distribution $Q$, seeking low error on $Q$ while maintaining a low abstention rate on $P$ \citep*{DBLP:conf/nips/GoldwasserKKM20,DBLP:conf/alt/KalaiK21}. This model is closely connected to reliable learning \citep*{DBLP:journals/jcss/KalaiKM12}, testable learning with distribution shifts \citep*{DBLP:conf/colt/KlivansSV24,DBLP:conf/colt/PatelKSV26}, and has extensions to sequential adversarial injections \citep*{DBLP:conf/nips/GoelHMS23,EdelmanGoel26,DBLP:conf/colt/YuB26}.

\section{Technical Overview}

\subsection{Sample and Oracle Complexity Upper Bound}
We highlight the main ideas behind the analysis of \prettyref{alg:bagging} and proof of \prettyref{thm:linearvc}. Our proof follows a substantially different route from \citet*{DBLP:conf/colt/MontasserHS19}. Instead of relying on sample compression arguments which incur a multiplicative factor of dual VC dimension $d^\star$ in sample complexity, we proceed with a more direct leave-one-out analysis of bagging RERMs.

Fix an arbitrary class $\Fc$ with VC dimension $d$, a perturbation map $\Uc: \X\to 2^\X$, an arbitrary $\RERM_\Fc$ oracle $\widehat{f}$ \eqref{eqn:rerm}, and a distribution $P$ over $\X\times \Y$ that is robustly realizable ($\inf_{f\in \Fc} \Risk(f;P)=0)$. The analysis is divided into three main parts we sketch below.

\paragraph{Distribution over RERMs.} The key idea is to the analyze the distribution over the RERMs $\widehat{f}_S$ that are produced from random samples $S\sim P^n$. In key \prettyref{lem:key}, we essentially show that in expectation over random test examples $(X, Y) \sim P$, for any perturbation $Z\in \Uc(X)$, the fraction of RERMs $\widehat{f}_S$ that misclassify $Z$, $\widehat{f}_S(Z)\neq Y$, is small:
\begin{equation}
    \label{eq:tech-1}
    \Ex{(X,Y)\sim P} \brk{ \sup_{Z \in \Uc(X)} \Ex{S\sim P^n} \brk{\1\set{\widehat{f}_S(Z) \neq Y}}  }=\tilde{O}\prn{\frac{d}{n}}.
\end{equation}
One way of thinking about \eqref{eq:tech-1} is that it swaps the order of $\sup$ and inner expectation relative to the expected robust risk of a single RERM, which does {\em not} vanish to zero in general as implied by \citep*[Theorem 1,][]{DBLP:conf/colt/MontasserHS19},
\[
    \Ex{(X,Y)\sim P}\brk{\Ex{S\sim P^n} \brk{\sup_{Z \in \Uc(X)}  \1\set{\widehat{f}_S(Z) \neq Y}}} = \Omega(1).
\]

\paragraph{Bagging Leave-One-Out Analysis.} The next idea is to apply \eqref{eq:tech-1} on an empirical distribution and establish a leave-one-out bound on the robust risk. Specifically, in \prettyref{lem:leave-one-out-margin}, on a fixed robustly realizable sequence $T=((X_i, Y_i))_{i=1}^{n}$, the aggregate vote over all RERMs can be expressed as
\begin{equation}
    \label{eq:tech-full-bagging}
    \widehat{\B}_T(x) := \Ex{S\sim {\rm Unif}(T)^n} \brk{\widehat{f}_{S}(x)} \in [-1, 1].
\end{equation}
\eqref{eq:tech-full-bagging} can be thought of as performing bagging (bootstrap aggregation) over all ordered bootstraps. This idealized full-bagging predictor is analogous to a voting predictor that is analyzed by \citet*{DBLP:conf/colt/Larsen23} for classical PAC learning, though our analysis follows a fundamentally different route tailored to the robust risk relying on \eqref{eq:tech-1} and the leave-one-out analysis we sketch below.

Consider a uniform distribution $D$ over $T$ and $D_{-i}$ a uniform distribution on $T_{-i}$ (leaving out the $i^{\rm th}$ tuple) for each $i \in [n]$. Note that the probability the $i^{\rm th}$ tuple is left out when drawing $n-1$ samples from $D$ is equal to $(1-1/n)^{n-1}\geq 1/e$. Combining this with \eqref{eq:tech-1}, it follows that
\[
    \frac{1}{n}\sum_{i=1}^{n} \sup_{Z_i\in \Uc(X_i)} \Pr_{S\sim D_{-i}^{n-1}} \brk{\widehat{f}_S(Z_i) \neq Y_i} \leq e \cdot \frac{1}{n}\sum_{i=1}^{n} \sup_{Z_i\in \Uc(X_i)} \Pr_{S\sim D^{n-1}} \brk{\widehat{f}_S(Z_i) \neq Y_i} =\tilde{O}\prn{\frac{d}{n-1}}.
\]
By applying Markov's inequality, we establish the following leave-one-out robust error bound on the full bagging predictor \eqref{eq:tech-full-bagging}
\begin{equation}
    \label{eq:tech-leave-one-out}
    \frac{1}{n} \sum_{i=1}^{n} \sup_{Z_i \in \Uc(X_i)} \1\set{Y_i\widehat{\B}_{T_{-i}}(Z_i) \leq 0} \leq 2\cdot \frac{1}{n}\sum_{i=1}^{n} \sup_{Z_i\in \Uc(X_i)} \Pr_{S\sim D_{-i}^{n-1}} \brk{\widehat{f}_S(Z_i) \neq Y_i} = \tilde{O}\prn{\frac{d}{n-1}}.
\end{equation}
We prove a more general version of this bound that provides an additional margin guarantee in \prettyref{lem:leave-one-out-margin}, which will be useful for us later in a sparsification argument. Observe that \eqref{eq:tech-leave-one-out}, via an exchangeability argument, implies the following bound on the expected robust risk of \eqref{eq:tech-full-bagging},
\[
    \Ex{S\sim P^n} \brk{\Risk\prn{{\rm sign}\prn{\widehat{B}_S}; P}} = \tilde{O}\prn{\frac{d}{n}}.
\]

\paragraph{High-Probability Bound and Sparsification.} We next apply a suffix averaging technique due to \citet*{DBLP:conf/focs/Aden-AliCSZ23} to extend the bound in \eqref{eq:tech-leave-one-out} to a high probability bound. With the help an additional technical lemma (\prettyref{lem:margin-average}), we establish in \prettyref{lem:full-agg} a high-probability bound on the robust risk of the {\em suffix-bagging} predictor $\widehat{g}_S: \X \to [-1,1]$ defined as
\begin{equation}
    \label{eq:tech-suffix-agg}
    \widehat{g}_S(x) := \frac{4}{3n} \sum_{t=n/4}^{n-1} \widehat{B}_{S_{\leq t}}(x) = \Ex{t\sim\Unif(J_n)}\brk{\Ex{S'\sim \Unif(S_{\leq t})^t} \brk{\widehat{f}_{S'}}}.
\end{equation}
Observe that the suffix-bagging predictor \eqref{eq:tech-suffix-agg} computes the aggregate vote over a distribution over RERMs $\widehat{f}_{S'}$ where $S'$ is sampled according to the process defined in \eqref{eq:tech-suffix-agg}.

Finally, note that the suffix-bagging predictor \eqref{eq:tech-suffix-agg} is {\em intractable} to compute. To overcome this challenge, we observe that we can approximate \eqref{eq:tech-suffix-agg} via a sparsification technique (\prettyref{lem:sparsification}) due to \citet*{DBLP:journals/jacm/MoranY16} by drawing {\em enough} bootstrap samples $S'_1, \dots, S'_N$ according to the process defined in \eqref{eq:tech-suffix-agg}. \prettyref{lem:sparsification} appeals to uniform convergence over the dual class of $\Fc$ to show that $N=O(d^\star)$ bootstrap samples suffice for approximation, where $d^\star$ denotes the dual VC dimension. Observe the bagging algorithm performs exactly this sampling process of bootstraps to produce its majority-vote output. This is where the oracle complexity dependence on $d^\star$ shows up.

\subsection{Oracle Complexity Lower Bound}

\paragraph{Proof Overview.}
The proof constructs a finite hard family $\{(\Fc_{\pi},P_{\pi,t})\}_{\pi,t}$ and places a prior on it by choosing the target index $T\in[K]$ uniformly and independently permuting $N$ copies of a latent instance set (which defines the randomized function class $\Fc_\pi$).  The lower bound rests on two distinct forms of hidden information.  First, even after observing an arbitrarily large training sample and making fewer than $d^\star$ oracle calls, the learner is unlikely to have the zero-robust-risk target among the hypotheses returned by the oracle.  Second, the returned hypotheses reveal too little information to determine the orientation of a carefully chosen opposite-label pair on a uniformly random unseen block.  The second obstruction applies even to an improper learner that combines and evaluates all returned hypotheses arbitrarily.

\paragraph{The Function Class and the Role of \(\pi\).} For suitably chosen integer parameters $K, N$, set \(B=2^d-1\), so \(d^\star=2B+1\), and consider the latent instance set
\[
    \Theta=\{(b,J):b\in\{-1,+1\},\ J\subseteq[K],\ |J|\le B\}.
\]
The (randomized) class $\Fc$ will consist of $K$ hypotheses. Hypothesis \(f_i\) labels \((b,J)\) by \(-b\) if \(i\in J\), and by \(b\) otherwise. The instance space $\X$ will consist of $N$ disjoint blocks, where each block \(Z_r\) is merely a copy of $\Theta$: a bijection \(\pi_r:Z_r\to\Theta\) determines the behavior of $\Fc$ on block $Z_r$. The permutations \(\pi_1,\ldots,\pi_N\) are chosen independently and uniformly, so the class behaves identically on every block, up to an unknown permutation of the rows.  Thus all combinatorial properties are determined by the behavior on \(\Theta\), while \(\pi\) hides where the relevant rows occur. In \prettyref{lem:dimensions}, we show that the (randomized) function class $\Fc$ has VC dimension $d$ and dual VC dimension $d^\star=2^{d+1}-1$.

\paragraph{Robust Realizability and the Hidden Target.}
A distribution $P_{\pi,t}$ draws a block $R \in [N]$, a label $Y$ uniformly, and a random set $A\subseteq[K]\setminus\{t\}$, and places in the perturbation set all types $(Y,J)$ with $J\subseteq A$.  The resulting robust loss has the exact form
\[
    \sup_{z\in\Uc(X)}\1\{f_i^\pi(z)\neq Y\}=\1\{i\in A\}.
\]
Since $t$ is always excluded from $A$, $f_t^\pi$ has zero robust risk. In \prettyref{lem:target-return}, we argue that after observing the training sample, every index absent from all sampled sets \(A_j\) remains a possible target, and conditional on the training sample the true target is uniform over this version space.  The class size \(K\) is chosen so that the version space remains large even when the sample size \(m\) is arbitrary.  Since \(q=d^\star-1\) oracle calls can return at most \(q\) distinct hypotheses, the probability that the returned set contains the target is less than a small constant \(\eta=10^{-3}\).

\paragraph{The Dual-VC Obstruction.}
Let \(Q\) be the set of indices of hypotheses returned by the $\RERM$ oracle.  Since \(|Q|\le q=2B\), split \(Q\) into two parts \(Q^+\) and \(Q\setminus Q^+\), each of size at most \(B\), and define
\[
    \theta^+=(+1,Q^+),
    \qquad
    \theta^-=(-1,Q\setminus Q^+).
\]
Both are valid instances \(\theta^+, \theta^- \in \Theta\). By construction of $\Fc$, one can verify that all returned hypotheses indexed by $Q$ predict the same label on \(\theta^+\) and \(\theta^-\) \eqref{eq:returned-agree}, whereas every unreturned hypothesis with index in $[K]\setminus Q$ labels \(\theta^+\) by \(+1\) and \(\theta^-\) by \(-1\) \eqref{eq:unreturned-opposite}.  This is the decisive role of the dual VC dimension: fewer than \(d^\star=2B+1\) returned hypotheses necessarily leaves an opposite-label pair that those hypotheses cannot distinguish. On each unseen block during training, the random permutation \(\pi_r\) hides which of two instances in $Z_r$ is assigned to \(\theta^+\) and which is assigned to \(\theta^-\); this unknown assignment is the orientation bit.

\paragraph{The Hidden Orientation is Nearly Uniform.}
The preceding local symmetry is not sufficient by itself. The response of the $\RERM$ oracle is adaptive and may depend on all hidden block permutations, including through arbitrary tie-breaking, so the identities of the returned hypotheses can leak information about the orientations.  After conditioning on the target, the learner's randomness, the training sample, and all sampled-block permutations, the unseen permutations remain independent.  On each unseen block, reveal every returned hypothesis and the unordered pair of instance assigned to \(\theta^+\) and \(\theta^-\).  Swapping the two assignments preserves all revealed information and flips only their orientation, so the orientation is locally uniform. \prettyref{lem:hidden-bit} then implies that the posterior orientation bias \(\Delta_R\) on a uniformly selected unseen test block satisfies
\[
    \E\!\left[\mathbf 1\{R \text{ is unseen in training}\}\Delta_R\right]<\eta.
\]
Thus, despite the entire adaptive interaction, the relevant orientation on a fresh block remains nearly uniform.

\paragraph{From Hidden Orientation to Robust Error.}
For the analysis, we condition on the training sample (and the revealed sampled-block permutations), the learner's internal randomness, the entire sequence of $\RERM$ oracle responses and on the values of every returned hypothesis on the whole domain. This information fixes the, possibly improper, classifier \(H\) outputted by the learner.  The relevant orientation remains hidden because the two special points are indistinguishable to all returned hypotheses.  On the {\em good} event that: the test block is unseen, the target was not returned, and \(Q\subseteq A\) \eqref{eq:good-event}, the point carrying \(\theta^+\) is a valid perturbation when \(Y=+1\), while the point carrying \(\theta^-\) is valid when \(Y=-1\).  For any fixed values of \(H\) at the two points, a uniform label and a uniform orientation produce an error in exactly two of the four possible cases, and hence with probability \(1/2\) \eqref{eq:fair-orientation-half}. The posterior bias can reduce this probability by at most \(\Delta_R\) \eqref{eq:posterior-error}.  The good event has probability close to one \eqref{eqn:eventlowerbound}, so averaging gives expected robust risk greater than \(0.49\).  Finally, averaging over the prior fixes a single hard pair \((\Fc_\pi,P_{\pi,t})\), and boundedness of the risk yields the stated constant-probability lower bound.

\section{Proof of \texorpdfstring{\prettyref{thm:linearvc}}{Theorem \getrefnumber{thm:linearvc}}: Sample and Oracle Complexity Upper Bound}

In this section, we provide a complete proof of \prettyref{thm:linearvc} which we restate below.

\realizablethm*

\subsection{Leave-One-Out Analysis of Bagging}

We begin with proving the following key lemma which establishes an important property: the fraction of RERMs that misclassify a worst-case perturbation of a random example is small. 

\begin{lemma}
    \label{lem:key}
    For any class $\Fc$ with VC dimension $d$, any $\Uc$, any deterministic $\RERM$ oracle $\widehat{f}$,
    any distribution $P$ over $\X \times \Y$ such that $\inf_{f\in \Fc} \Risk(f; P) = 0$, and any $\theta \in (0,1)$,
    \[
        \Pr_{(X,Y)\sim P} \brk{ a(X, Y) \geq \theta } \leq \frac{c}{\theta^2} \cdot \frac{d}{n}, \qquad a(x,y):= \sup_{z\in \Uc(x)}~~\Pr_{S\sim P^n} \brk{\widehat{f}_S(z) \neq y}.
    \]
\end{lemma}

To prove \prettyref{lem:key}, we make use of the following recent second-moment bound for consistent VC rules. We note that we could also rely on classical first-moment bounds \citep*{DBLP:journals/jacm/BlumerEHW89}, at the expense of an additional $\log(n/d)$ factor in the resulting bound.

\begin{lemma} [\cite*{DBLP:conf/colt/Aden-AliHLZ24,DBLP:journals/corr/abs-2606-13614}]
    \label{lem:vc-moment-2}
    Let $P$ be a distribution on $\X \times \Y$ and let $\Cc$ be a class of subsets of $\X \times \Y$ with VC dimension $d$.
    Let $C$ be any deterministic selector satisfying $C(S)\in \Cc$ for every $S\in (\X \times \Y)^n$ and $\Pr_{S\sim P^n}\brk{C(S)\cap S=\emptyset}=1$. Then,
    \[
        \underset{(X,Y) \sim P}{\E}\brk{ \Pr_{S\sim P^n} \brk{(X,Y)\in C(S)}^2 } \leq c \frac{d}{n}.
    \]
\end{lemma}

\begin{proof} [Proof of \prettyref{lem:key}]
    By Markov's inequality, observe that
    \[
        \Pr_{(X,Y)\sim P} \brk{ a(X, Y) \geq \theta } = \Pr_{(X,Y)\sim P} \brk{ a(X,Y)^2 \geq \theta^2 }
        \leq \frac{1}{\theta^2} \Ex{(X,Y)\sim P}\brk{a(X,Y)^2}.
    \]
    Thus, it suffices to show that
    \[
        \Ex{(X,Y)\sim P}\brk{a(X,Y)^2} = \Ex{(X,Y)\sim P} \brk{ \sup_{Z \in \Uc(X)} \Pr_{S\sim P^n} \brk{\widehat{f}_S(Z) \neq Y}^2  } \leq c \frac{d}{n}.
    \]
    We now proceed with analyzing the random variable $a(X,Y)$.
    Fix an arbitrary map $\phi: \X \to \X$ such that $\phi(x) \in \Uc(x)$ for all $x \in \X$.
    Define the collection of error sets
    \[
        \Cc_\phi = \set{\set{(x,y)\in \X \times \Y: f(\phi(x)) \neq y } \mid f \in \Fc}.
    \]
    Observe that $\vc(\Cc_\phi) \leq \vc(\Fc) = d$, because we are using a fixed map $\phi$.
    Consider the selector $C$ defined as follows
    \[
        C(S) = \set{(x,y)\in \X \times \Y: \widehat{f}_S(\phi(x)) \neq y}.
    \]
    Because $\widehat{f}$ is an $\RERM$ oracle for $\Fc$ and distribution $P$ is robustly realizable: $\inf_{f\in \Fc} \Risk(f; P)=0$,
    it follows that on any sample $S\sim P^n$, $\widehat{f}_S$ is robustly correct on $S$, i.e., for each $(X,Y)\in S$ and each $Z\in \Uc(X)$, $\widehat{f}_S(Z) = Y$. Since $\phi(x)\in \Uc(x)$, this implies that $\widehat{f}_S(\phi(X))= Y$ for each $(X,Y)\in S$. Therefore, $\Pr_{S\sim P^n}\brk{C(S)\cap S = \emptyset} = 1$. Also, since $\widehat{f}_S\in \Fc$, it follows that $C(S) \in \Cc_\phi$.

    We are now ready to invoke \prettyref{lem:vc-moment-2} to establish
    \begin{equation}
        \label{eqn:second-moment}
        \Ex{(X,Y)\sim P}\brk{\Pr_{S\sim P^n}\brk{\widehat{f}_S(\phi(X)) \neq Y}^2}
        = \Ex{(X,Y)\sim P}\brk{ \Pr_{S\sim P^n} \brk{(X,Y)\in C(S)}^2 } \leq c \frac{d}{n}.
    \end{equation}
    Observe that \eqref{eqn:second-moment} holds for any map $\phi:\X\to \X$ satisfying $\phi(x)\in \Uc(x)$ for all $x\in \X$. Fix an arbitrary $f^\star\in \Fc$ such that $\Risk(f^\star; P)=0$. For any scalar $\eta > 0$, we will consider a special arbitrary map $\phi_\eta: \X \to \X$ satisfying
    \[
        \phi_\eta(x) \in \Uc(x) \qquad \text{and}
        \qquad \Pr_{S\sim P^n} \brk{\widehat{f}_S(\phi_\eta(x)) \neq f^\star(x)}^2 \geq a\prn{x, f^\star(x)}^2 - \eta
        \qquad\text{for all }x\in\X.
    \]
    In words, $\phi_\eta$ maps each $x \in \X$ to a perturbation $\phi_\eta(x)=z\in \Uc(x)$ that
    is an approximate maximizer of the probability of disagreement $\Pr_{S\sim P^n}\brk{\widehat{f}_S(z)\neq f^\star(x)}$. Since $f^\star(X) =Y$ with probability $1$, it follows then by \eqref{eqn:second-moment} that
    \[
        \Ex{(X,Y)\sim P}\brk{a(X,Y)^2} \leq \Ex{(X,Y)\sim P}\brk{\Pr_{S\sim P^n}\brk{\widehat{f}_S(\phi_\eta(X)) \neq Y}^2} + \eta \leq c \frac{d}{n} + \eta.
    \]
    Taking $\eta$ arbitrarily close to zero concludes the proof.
\end{proof}

\paragraph{Bagging.} We introduce notation for bagging. For a sequence $S=((X_1, Y_1), \dots, (X_n, Y_n))$, let $S' \sim {\rm Unif}(S)^n$ be a bootstrap drawn uniformly at random (with replacement) from $S$ and let $\widehat{f}_{S'}$ be the output classifier of $\RERM$ \eqref{eqn:rerm} on $S'$. Equivalently, we can express this as drawing a sequence of integers $I=(i_1,\dots, i_n) \in [n]^n$ and denote the bootstrap sample by $S(I)=((X_{i_1}, Y_{i_1}), \dots, (X_{i_n}, Y_{i_n}))$.

Define the aggregate vote $\widehat{\B}_S: \X \to [-1,1]$ over all $n^n$ ordered bootstraps of $S$ as

\begin{equation}
    \label{eqn:full-bagging}
    \widehat{\B}_S(x) := \Ex{S'\sim {\rm Unif}(S)^n} \brk{\widehat{f}_{S'}(x)} = \frac{1}{n^n} \sum_{I \in [n]^n} \widehat{f}_{S(I)}(x).
\end{equation}

We now proceed to bound the leave-one-out robust risk of the bagging predictor \eqref{eqn:full-bagging}.

\begin{lemma}[Leave-One-Out Error]
    \label{lem:leave-one-out-margin}
    There is a universal constant $C> 0$ such that, for any class $\Fc$ with VC dimension $d$, any $\Uc$, any deterministic $\RERM$ oracle,
    any robustly realizable sequence $T = ((X_1, Y_1), \dots, (X_n, Y_n))$ with $n\geq 2$, and any $\gamma \in [0, 1)$,
    \[
        \frac{1}{n} \sum_{i=1}^{n} \sup_{Z_i \in \Uc(X_i)} \1\set{Y_i\cdot \widehat{\B}_{T_{-i}}(Z_i) \leq \gamma} \leq\frac{C}{(1-\gamma)^2}\cdot \frac{d}{n}.
    \]
\end{lemma}

\begin{remark}[Expected Robust Risk] \label{rem:expectationbound}
    Note that by an exchangeability argument, the leave-one-out-error in \prettyref{lem:leave-one-out-margin} translates to an expected {\em robust} risk bound of
    \[
        \Ex{S\sim P^n} \brk{\Ex{(X, Y)\sim P}\brk{\sup_{Z \in \Uc(X)} \1\set{ Y \cdot \widehat{\B}_{S} (Z) \leq \gamma }}} \leq \frac{C}{(1-\gamma)^2}\frac{d}{(n+1)}.
    \]
\end{remark}

\begin{proof}[Proof of \prettyref{lem:leave-one-out-margin}]
    \begingroup
    \setlength{\parskip}{1.75pt}
    \setlength{\abovedisplayskip}{5.1pt plus 1pt minus 2pt}
    \setlength{\belowdisplayskip}{5.1pt plus 1pt minus 2pt}
    \setlength{\abovedisplayshortskip}{3.1pt plus 1pt}
    \setlength{\belowdisplayshortskip}{4.1pt plus 1pt minus 1pt}
    Let $m=n-1$. Let $D={\rm Unif}(T)$ be a uniform distribution over $T$.
    Fix an arbitrary $(X_i, Y_i) \in T$ and an arbitrary $Z_i\in \Uc(X_i)$.
    Denote by $E_i$ the event that $(X_i, Y_i)$ appears in a random sample $S \sim D^m$.
    By law of total probability,
    \[
        \Pr_{S\sim D^{m}}\brk{\widehat{f}_S(Z_i) \neq Y_i} =
        \Pr\brk{E_i}\Pr\brk{\widehat{f}_S(Z_i) \neq Y_i \mid E_i} +
        \Pr\brk{\bar{E}_i}\Pr\brk{\widehat{f}_S(Z_i) \neq Y_i \mid \bar{E}_i}.
    \]
    Because $\widehat{f}$ is an $\RERM$ oracle for $\Fc$ and because distribution $D$
    is robustly realizable, it follows that for any sample $S\sim D^n$,
    $\widehat{f}_S$ is robustly correct on $S$: for each $(X,Y)\in S$ and each $Z\in \Uc(X)$, $\widehat{f}_S(Z) = Y$.
    Under event $E_i$, $(X_i, Y_i)\in S$, thus, $\widehat{f}_S(Z_i)=Y_i$. This implies that
    \[
        \Pr_{S\sim D^{m}}\brk{\widehat{f}_S(Z_i) \neq Y_i} =
        \Pr\brk{\bar{E}_i}\Pr\brk{\widehat{f}_S(Z_i) \neq Y_i \mid \bar{E}_i} =
        \prn{1 - \frac{1}{n}}^{m} \Pr\brk{\widehat{f}_S(Z_i) \neq Y_i \mid \bar{E}_i}.
    \]

    Let $D_{-i}={\rm Unif}(T_{-i})$ be a uniform distribution on $T_{-i}$.
    Observe that sampling $S\sim D^m$ conditioned on event $E_i$ is equivalent to
    sampling $\tilde{S}\sim D^m_{-i}$,
    \[
        \Pr_{S\sim D^m}\brk{\widehat{f}_S(Z_i) \neq Y_i \mid \bar{E}_i} = \Pr_{\tilde{S}\sim D_{-i}^m}\brk{\widehat{f}_{\tilde{S}}(Z_i) \neq Y_i}.
    \]
    Combining the above, we establish that for each $(X_i, Y_i) \in T$ and each $Z_i \in \Uc(X_i)$,
    \begin{equation}
        \label{eqn:leave-out}
        \Pr_{S\sim D^m}\brk{\widehat{f}_S(Z_i) \neq Y_i} =
        \prn{1 - \frac{1}{n}}^{m} \Pr_{\tilde{S}\sim D_{-i}^m}\brk{\widehat{f}_{\tilde{S}}(Z_i) \neq Y_i} \geq
        \frac{1}{e} \Pr_{\tilde{S}\sim D_{-i}^m}\brk{\widehat{f}_{\tilde{S}}(Z_i) \neq Y_i},
    \end{equation}
    where the last inequality follows from the fact that for $m=n-1$, $(1-1/n)^{n-1} \geq 1/e$ for all $n \geq 2$.

    Fix an arbitrary $(X_i, Y_i)\in T$ and consider the averaging classifier over all bootstraps $\widehat{\B}_{T_{-i}}$ computed on $T_{-i}$. Observe that if
    \[
        \sup_{Z_i \in \Uc(X_i)} \1\set{ Y_i \cdot \widehat{\B}_{T_{-i}} (Z_i) \leq \gamma } = 1,
    \]
    then, by definition, there exists $Z_i \in \Uc(X_i)$ such that
    \[
        \Pr_{\tilde{S}\sim D_{-i}^m}\brk{\widehat{f}_{\tilde{S}}(Z_i) \neq Y_i} = \frac{1}{2}\prn{1-Y_i\widehat{\B}_{T_{-i}}(Z_i)} \geq \frac{1-\gamma}{2}.
    \]
    Combining this observation with \eqref{eqn:leave-out}, we can relate the performance of $\widehat{\B}_{T_{-i}}$ with that of the performance of $\widehat{\B}_T$ as follows
    \[
        \sup_{Z_i\in\Uc(X_i)} \1\set{ Y_i \cdot \widehat{\B}_{T_{-i}} (Z_i) \leq \gamma } \leq \sup_{Z_i \in \Uc(X_i)} \1\set{ \Pr_{S\sim D^{m}}\brk{\widehat{f}_S(Z_i) \neq Y_i} \geq  \frac{1-\gamma}{2e} }
        = \1\set{a(X_i, Y_i) \geq \frac{1-\gamma}{2e}}.
    \]
    By summing over $i=1,\dots, n$ and invoking \prettyref{lem:key} with $\theta = (1-\gamma)/(2e)$, we obtain
    \[
        \frac{1}{n} \sum_{i=1}^{n} \sup_{Z_i \in \Uc(X_i)} \1\set{ Y_i \cdot \widehat{\B}_{T_{-i}} (Z_i) \leq \gamma }
        \leq \Pr_{(X,Y) \sim D} \brk{a(X,Y) \geq \frac{1-\gamma}{2e}} \leq c\cdot 4e^2 \cdot \frac{1}{(1-\gamma)^2} \cdot \frac{d}{m}.
    \]
    Since $m=n-1$ and $n\geq 2$, it follows that $n/(n-1) \leq 2$. This concludes the proof.
    \endgroup
\end{proof}

\subsection{High-Probability Bound and Sparsification}

Our goal here is to convert the leave-one-out-error bound of \prettyref{lem:leave-one-out-margin} to a high probability bound.

\paragraph{Suffix Averaging.} Our first step is to apply the suffix averaging technique due to \cite*{DBLP:conf/focs/Aden-AliCSZ23} on top of bagging \eqref{eqn:full-bagging}. Specifically, for a sequence $S=((X_1, Y_1), \dots, (X_n, Y_n))$ where $n/4$ is an integer, let $J_n=\{n/4,\dots, n-1\}$. For $t\in J_n$, let $S_{\leq t}=((X_i, Y_i))_{i\leq t}$ be a suffix of length $t$, and let $\widehat{B}_{S_{\leq t}}$ be the bagging predictor produced on $S_{\leq t}$.

The proof of the following lemma is a straightforward extension of the suffix averaging argument of \citet*[][Theorem 2.1]{DBLP:conf/focs/Aden-AliCSZ23} applied to our margin-robust loss. We defer it to \prettyref{app:appendix1}.

\begin{lemma}[High-probability suffix average]
    \label{lem:high-prop}
    Suppose that $n/4$ is an integer. Then, for every
    distribution $P$ over $\X\times \Y$ satisfying $\inf_{f^\star \in \Fc} \Risk(f^\star; P)=0$ (i.e., robustly realizable), every $\delta\in(0,1)$, and $S\sim P^n$, with probability at least $1-\delta$,
    \[
        \frac{4}{3n} \sum_{t=n/4}^{n-1} \Ex{(X,Y)\sim P} \brk{ \sup_{Z\in\Uc(X)} \1\set{Y\cdot \widehat{\B}_{S_{\leq t}}(Z)\leq\gamma}}
        \leq
        4.82 C\prn{ \frac{d}{(1-\gamma)^2n} + \frac{1}{n}\log\frac{2}{\delta}}.
    \]
\end{lemma}

\paragraph{Suffix-Bagging Predictor.} Observe that \prettyref{lem:high-prop} is not enough by itself as it only guarantees that the average of the margin-robust risk over the suffixes is small. We introduce next the {\em suffix-bagging predictor} which averages all the suffix predictors and then proceed to bound its margin-robust risk. Formally, define the suffix-bagging predictor $\widehat{g}_S: \X \to [-1,1]$ as
{\vskip -5pt}
\begin{equation}
    \label{eqn:suffix-agg}
    \widehat{g}_S(x) := \frac{4}{3n} \sum_{t=n/4}^{n-1} \widehat{B}_{S_{\leq t}}(x) = \Ex{t\sim J_n}\brk{\Ex{S'\sim {\rm Unif}(S_{\leq t})^t} \brk{\widehat{f}_{S'}}} = \frac{1}{|J_n|} \sum_{t=n/4}^{n-1} \frac{1}{t^t} \sum_{I\in[t]^t} \widehat{f}_{S(I)}(x).
\end{equation}

The next lemma helps us bound the robust-margin loss of the suffix-bagging predictor \eqref{eqn:suffix-agg} by the average robust-margin loss of the suffixes.

\begin{lemma}[Averaging two margins]
    \label{lem:margin-average}
    Let $g_1,\ldots,g_k:\X\to[-1,1]$ and $\bar g:=(1/k)\sum_{j=1}^k g_j$.  For every $0\leq\gamma<\rho<1$ and every $(x,y)$,
    \begin{equation}
        \label{eqn:margin-average}
        \sup_{z\in\Uc(x)}\1\set{y\bar g(z)\leq \gamma} \leq \frac{1+\rho}{\rho-\gamma}\cdot \frac1k\sum_{j=1}^k \sup_{z\in\Uc(x)} \1\set{yg_j(z)\leq \rho}.
    \end{equation}
\end{lemma}

\begin{proof}
    Suppose $\sup_{z\in\Uc(x)}\1\set{y\bar g(z)\leq \gamma} = 1$, and choose $z\in\Uc(x)$ such that
    $y\bar g(z)\leq\gamma$.  Let $q$ be the fraction of indices $j \in[k]$ for which $y g_j(z)\leq\rho$.  Since every margin lies in $[-1,1]$,
    \[
        y\bar g(z) = \frac{1}{k}\sum_{j: yg_j(z)\leq \rho} yg_j(z) + \frac{1}{k}\sum_{j: yg_j(z)> \rho} yg_j(z) \geq -q+(1-q)\rho
        =\rho-(1+\rho)q.
    \]
    Since $y\bar g(z)\leq \gamma$, by rearranging terms, it follows that $q\geq(\rho-\gamma)/(1+\rho)$. By definition of $q$ and since $z\in \Uc(x)$, it follows that
    \[
        \frac1k\sum_{j=1}^k \sup_{z\in\Uc(x)} \1\set{yg_j(z)\leq \rho} \geq q \geq \frac{\rho - \gamma}{1 + \rho}.
    \]
    Rearranging terms concludes the proof.
\end{proof}

We now bound the margin-robust risk of the suffix-bagging predictor \eqref{eqn:suffix-agg} with high-probability.

\begin{lemma}
    \label{lem:full-agg}
    For any class $\Fc$ with finite VC dimension $d$, any $\Uc$, any (deterministic) $\RERM$ oracle $\widehat{f}$,
    any distribution $P$ over $\X\times \Y$ satisfying $\inf_{f^\star \in \Fc} \Risk(f^\star; P)=0$ (i.e., robustly realizable), any $n\geq 1$ and confidence parameter $\delta\in (0,1)$, with probability $1-\delta$ over $S\sim P^n$, it holds that
    \[
        \Ex{(X, Y)\sim P}\brk{\sup_{Z \in \Uc(X)} \1\set{ Y \cdot \widehat{g}_S(Z) \leq 1/2 }} \leq C_2\prn{\frac{d}{n} + \frac{1}{n}\log\prn{\frac{2}{\delta}}}.
    \]
\end{lemma}

\begin{proof}
    We start by invoking \prettyref{lem:high-prop} to convert the in leave-one-out error bound of \prettyref{lem:leave-one-out-margin} to a high probability bound, with margin parameter set to a constant equal to $3/4$. It follows that with probability at least $1-\delta$ over $S\sim P^n$,
    \[
        \frac{1}{|J_n|}\sum_{t\in J_n} \Ex{(X, Y)\sim P}\brk{\sup_{Z \in \Uc(X)} \1\set{ Y \cdot \widehat{\B}_{S_{\leq t}} (Z) \leq 3/4 }} \leq 64C_1 \prn{\frac{d}{n} + \frac{1}{n}\log\prn{\frac{2}{\delta}}}.
    \]
    To finish the proof, we invoke \prettyref{lem:margin-average} with margin parameters $\gamma=1/2, \rho=3/4$ and apply to the classifiers $g_t := \widehat{\B}_{S_\leq t}$ for $t\in J_n$. By definition \eqref{eqn:suffix-agg}, $\widehat{g}_S = 1/|J_n| \sum_{t\in J_n} g_t$. The resulting constant $C_2=(1+\rho)/(\rho-\gamma)64C_1=7\cdot64C_1= 448C_1$.
\end{proof}

Note that the suffix-bagging predictor \eqref{eqn:suffix-agg} is {\em intractable} to compute. To overcome this challenge, we observe that we can approximate \eqref{eqn:suffix-agg} via a sparsification technique due to \citet*{DBLP:journals/jacm/MoranY16} which we state in the lemma below.

\begin{lemma}[Sparsification, \citet*{DBLP:journals/jacm/MoranY16}]
    \label{lem:sparsification}
    Let $\Fc$ be a function class with finite {\em dual} VC dimension $d^\star$. Let $Q$ be an arbitrary distribution over $\Fc$. Then, for any $\eps, \delta \in (0,1)$, letting $N=O\prn{\frac{d^\star + \log(1/\delta)}{\eps^2}}$, with probability at least $1-\delta$ over $f_1,\dots, f_N \sim Q$,
    \[
        \forall x \in \X,\qquad \left| \frac{1}{N} \sum_{j=1}^{N} f_j(x) - \Ex{f\sim Q}\brk{f(x)} \right| \leq \eps.
    \]
\end{lemma}

\subsection{Putting the Pieces Together}

\begin{proof}[Proof of \prettyref{thm:linearvc}]

    We start by presenting the learning guarantee for the idealized suffix-bagging-aggregate predictor \eqref{eqn:suffix-agg}. By invoking \prettyref{lem:full-agg}, with probability at least $1 -\delta/2$ over $S\sim P^n$, it holds that
    \begin{equation}
        \label{eqn:bound1}
        \Ex{(X,Y)\sim P} \brk{\sup_{Z\in \Uc(X)} \1\set{Y \cdot \widehat{g}_S(Z) \leq 1/2}} \leq C_2\prn{\frac{d}{n} + \frac{1}{n}\log\prn{\frac{4}{\delta}}}.
    \end{equation}

    By definition \eqref{eqn:suffix-agg}, observe that the suffix-bagging-aggregate predictor $\widehat{g}_S$ can be viewed as a distribution $Q$ over predictors in the class $\Fc$. We can sample from this distribution $Q$ by sampling an index $t$ uniformly at random from $J_n$ and then sampling uniformly at random a bootstrap sample $S'$ from $S_{\leq t}$. By invoking \prettyref{lem:sparsification} with an approximation parameter $\eps = 1/4$, letting $N=O(d^\star + \log(4/\delta))$, we have with probability $1-\delta/2$ over the bootstrap samples $S'_1,\dots, S'_N \subseteq S$,
    \begin{equation}
        \forall x\in \X,\qquad \left| \frac{1}{N} \sum_{j=1}^{N} \widehat{f}_{S'_j}(x) - \widehat{g}_S(x)   \right| \leq \frac{1}{4}.
    \end{equation}
    This implies the following guarantee
    \begin{equation}
        \label{eqn:sparse-margin}
        \forall (z, y)\in \X \times \Y,\qquad \text{ if } y \widehat{g}_S(z) > \frac{1}{2}
        \text{ then }y \prn{\frac{1}{N} \sum_{j=1}^{N} \widehat{f}_{S'_j}(z)} > \frac{1}{4}.
    \end{equation}
    Combining \eqref{eqn:bound1} and \eqref{eqn:sparse-margin}, with probability at least $1-\delta$ over the random draw of a training dataset $S\sim P^n$ and $N$ bootstrap samples $S'_1\dots, S'_N \subseteq S$,
    \begin{align}
        \Risk\prn{ \MAJ(\widehat{f}_{S'_1},\dots, \widehat{f}_{S'_N}); P}
        &\leq \Ex{(X,Y)\sim P} \brk{\sup_{Z\in \Uc(X)} \1\set{Y \prn{\frac{1}{N} \sum_{j=1}^{N} \widehat{f}_{S'_j}(Z)}\leq 1/4}}\\
        &\leq  \Ex{(X,Y)\sim P} \brk{\sup_{Z\in \Uc(X)} \1\set{Y \cdot \widehat{g}_S(Z) \leq 1/2}}\\
        &\leq C_2\prn{\frac{d}{n} + \frac{1}{n}\log\prn{\frac{4}{\delta}}}.
    \end{align}
    This concludes the proof.
\end{proof}

\section{Proof of \texorpdfstring{\prettyref{thm:lowerbound}}{Theorem \getrefnumber{thm:lowerbound}}: Oracle Complexity Lower Bound}

In this section, we provide a complete proof of \prettyref{thm:lowerbound} which we restate below.

\oraclelowerboundthm*

\subsection{Construction}

\paragraph{Parameters.} Let $d, m \geq 1$. We construct below a family of function classes, each of which has VC dimension $d$ and dual VC dimension $d^\star=2^{d+1} - 1$. Parameter $m$ denotes the number of samples provided to the learning algorithm which can be arbitrarily large, and the number of queries the algorithm can make to an RERM oracle is capped to $q= d^\star - 1$.

Set auxiliary parameters $B=(d^\star-1)/2 = 2^{d}-1, \eta = 10^{-3}, \delta = \eta /d^\star, p = 1-\delta$. Let $K=\ceil{\frac{d^\star}{\eta \delta^m}}$, and $N= m + \ceil{\frac{m}{\eta} + \frac{q\log(K+1)}{2\eta^2}}$. $K$ will be the cardinality of each of the function classes, $N$ plays a role in the cardinality of the instance space. Later in the proof we will make use of the following inequalities which can be verified given the choice of parameters:
\begin{equation}
    \label{eq:param-constraint}
    K \geq d^\star/\eta > d^\star > 2^d,\qquad \frac{q}{K\delta^m} < \eta,\qquad \frac{m}{N} < \eta,\qquad \sqrt{\frac{q\log(K+1)}{2(N-m)}} < \eta.
\end{equation}

\paragraph{Instance Space.} Define the finite set of latent instances
\[
    \Theta := \set{(b, J): b\in \set{-1,1}, J\subseteq [K], |J| \leq B}.
\]
For every $r \in [N]$, let $Z_r$ be an abstract set of cardinality $|\Theta|$, with the $Z_r$'s being pairwise disjoint. We a priori fix an ordering of $Z_r$. For every $V\subseteq Z_r$, introduce a center instance $c_{r, V}$. Define the instance space
\[
    \X := \prn{\cup_{r=1}^{N} Z_r} \cup \set{c_{r, V}: r\in [N], V\subseteq Z_r }.
\]

\paragraph{Perturbations.} Define the perturbation map
\[
    \Uc(z):= \set{z}\quad (z\in Z_r), \quad \Uc(c_{r, V}):= V.
\]

\paragraph{Family of Function Classes.} We will now define the family of function classes based on the set of latent instances $\Theta$ and a bijection tuple $\pi=(\pi_1,\dots, \pi_N)$ where each $\pi_r: Z_r \to \Theta$ is a bijection map. For each $\pi$, the function class $\Fc_\pi =\set{f^\pi_1,\dots, f^\pi_K}$ is defined as follows. If $\pi_r(z) = (b, J)$, define for every $i\in [K]$,

\[
    f_i^\pi(z)
    :=
    \begin{cases}
        -b,&i\in J,\\
        b,&i\notin J,
    \end{cases}
    \qquad
    f_i^\pi(c_{r,V}):=+1.
\]
Note that the functions $f^\pi_i$ are distinct since the latent instance $(+1, \set{i})$ separates $f^\pi_i$ from every $f^\pi_j$, $j\neq i$. Note also that all functions in $\Fc_\pi$ label the center instances $c_{r, V}$ with $+1$, the nontrivial action happens only in the blocks $Z_r, r\in [N]$.

This defines the family of function classes $\set{\Fc_\pi}_{\pi}$ we will use in the lower bound. In the proof, we will draw the block bijections $\pi=(\pi_1,\dots, \pi_N)$ independently and uniformly at random. It may be helpful to think of the (random) function class $\Fc^{\Pi}$ as a matrix with columns indexed by $1, \dots, K$, and rows indexed by the blocks $Z_1, \dots, Z_N$.

\begin{lemma}[Primal and Dual VC dimension]\label{lem:dimensions}
    For every \(\pi\), $\vc(\Fc_\pi)=d$ and $\vc^\star(\Fc_\pi)=d^\star = 2^{d+1} - 1$.
\end{lemma}

\begin{proof}

    We start by showing $\vc(\Fc_\pi) \geq d$. Choose \(2^d=B+1\) distinct indices
    \[
        \{i_s:s\in\{-1,+1\}^d\}\subseteq[K].
    \]
    For \(j\in[d]\), let
    \[
        J_j:=\{i_s:s_j=-1\}.
    \]
    Then \(|J_j|=2^{d-1}\leq B\).  In one block \( \pi_r: Z_r \to \Theta \), choose the instances \(z_j \in Z_r\) that map to \((+1,J_j)\).  For every \(s\in\{-1,+1\}^d\), $f_{i_s}^\pi(z_j)=s_j$. Hence \(z_1,\ldots,z_d\) are shattered.

    We now show that $\vc(\Fc_\pi)\leq d$. Suppose that \(d+1\) instances were shattered.  Choose one function for each of the $2^{d+1}=2B+2$ label vectors on these points.  At each selected instance, exactly half of
    these functions, namely \(B+1\), must be negative. But, by construction of the class, this is impossible. Specifically, on any
    \(2B+2\) functions in $\Fc_\pi$, a latent instance $(b, J)\in \Theta$ has at most \(B\) negative values (if $b=1$) or at least
    \(B+2\) negative values (if $b=-1$), because $|J|\leq B$ and $K - |J| \geq B+2$.  Therefore \(\vc(\Fc_\pi)=d\).

    We proceed to showing $\vc^\star(\Fc_\pi)=d^\star = 2B+1$. For the lower bound, choose any \(G\subseteq[K]\) with $|G|=d^\star=2B+1$. Given \(J\subseteq G\), if \(|J|\leq B\), the latent instance \((+1,J)\) is labeled negative by every $f^\pi_i, i\in J$ and positive by every $f^\pi_i, i \in G \setminus J$. If \(|J|\geq B+1\), then \(|G\setminus J|\leq B\), and the latent instance \((-1,G\setminus J)\) is labeled negative by every $f^\pi_i, i\in J$ and positive by every $f^\pi_i, i \in G \setminus J$. Hence, the dual class shatters $G$.

    For the upper bound, we claim that the dual class cannot shatter \(d^\star+1=2B+2\) functions.  Indeed, on such a set, a subset $G$ of size $B+1$ can not be labeled negative since it is neither small nor the complement of a small
    set, so no latent instance in $\Theta$ can realize this labeling.  Thus
    \(\vc^\star(\Fc_\pi)=d^\star\).
\end{proof}

\paragraph{The Family of Distributions.} For every $\pi$ and $t\in[K]$, define a distribution $P_{\pi,t}$ over $\X \times \Y$ as follows. A sample $(X, Y)\sim P_{\pi, t}$ is generated by:
\begin{enumerate}[label=(\roman*)]
    \item Draw $R\sim\Unif([N])$ and $Y\sim\Unif(\Y)$ independently.
    \item Form $A\subseteq[K]\setminus\{t\}$ by including each
    $i\neq t$ independently with probability $p$;
    \item Form the set of perturbations $V_\pi(R,A,Y)
    :=\{z\in Z_R:\pi_R(z)=(Y,J)\text{ for some }J\subseteq A\}$ and set $X=c_{R,V_\pi(R,A,Y)}$.
    \item Output $(X,Y)$.
\end{enumerate}

In the lower bound proof, we will be using the collection $\set{(\Fc_\pi, P_{\pi, t}): \pi, t\in [K]}$. The family of distributions $\set{P_{\pi, t}}_{t\in [K]}$ is {\em robustly} realizable by the class $\Fc_\pi$ as given by the following lemma.

\begin{lemma}[Robust Realizability]\label{lem:loss-identity}
    Let $(X, Y) \sim P_{\pi, t}$ with corresponding latent random variables $(R,A,Y)$. For every $i\in[K]$,
    \begin{equation}\label{eq:loss-identity}
        \sup_{z\in \Uc(X)} \1\set{f^\pi_i(z)\neq Y} = \sup_{z\in V_{\pi}(R,A,Y)} \1\set{f^\pi_i(z)\neq Y} = \1\set{i \in A}.
    \end{equation}
    Consequently, $\Risk(f_t^\pi;P_{\pi,t})=0$.
\end{lemma}

\begin{proof}
    Recall that $V_\pi(R,A,Y)=\{z\in Z_R:\pi_R(z)=(Y,J)\text{ for some }J\subseteq A\}$.
    If $i\notin A$, then every $J\subseteq A$ excludes $i$, so
    $f_i^\pi(z)=Y$ for every $z\in V_\pi(R,A,Y)$.  Hence the robust loss is zero.  If
    $i\in A$, then there is a perturbation $z\in V_\pi(R,A,Y)$ such that $\pi_R(z)=(Y,\{i\})$, on which $f_i^\pi(z)=-Y$, so the robust loss is one.  Since the distribution $P_{\pi,t}$ always excludes $t$ from $A$, the target $f_t^\pi$ has zero robust risk.
\end{proof}

\paragraph{Lower Bound Analysis.} The bijection maps $\Pi=(\Pi_1, \dots,\Pi_N)$ are drawn independently and uniformly at random. This determines the random function class $\Fc_\Pi$. In addition, the target function is drawn independently $T\sim \Unif([K])$. This defines the random distribution $P_{\Pi, T}$. Then, a training sample $S=((X_j, Y_j))\sim P_{\Pi, T}^m$ is generated, with corresponding latent variables $((R_j, A_j))_{j=1}^{m}$.

Upon receiving a training sample $S \sim P_{\Pi, T}^m$ as input, the learner $\Alg$ makes at most $q$ queries to an $\RERM$ oracle for $\Fc_\Pi$. We denote by $I=(I_1,\dots, I_q) \in ([K]\cup\set{\perp})^q$ the ordered sequence of indices of the returned functions by the oracle where unused queries are padded with $\perp$. We denote by $Q=Q(I)\subseteq [K]$ the set of distinct indices, $|Q|\leq q$. We denote by $\Phi_I=(f^\Pi_{I_1}, \dots, f^{\Pi}_{I_q})$ the returned functions by the oracle. After its interaction and oracle queries, the learner $\Alg$ outputs an arbitrary predictor $H: \X \to \Y$. Note that $H$ depends on $S, I, \Phi_I$ and $\Alg$'s internal randomness. Our goal is to bound from below the expected robust risk
\[
    \Ex{\Pi, T} \Ex{S\sim P_{\Pi, T}^m} \Risk(H; P_{\Pi, T}) = \Ex{\Pi, T} \Ex{S\sim P_{\Pi, T}^m} \Ex{(X,Y)\sim P_{\Pi, T}} \brk{\sup_{Z\in \Uc(X)} \1\set{H(Z)\neq Y}}.
\]

\subsection{The Target is Unlikely to be Returned}
\begin{lemma}
    \label{lem:target-return}
    Let $Q\subseteq [K]$ be the set of distinct indices of all functions returned to a learner making at most $q$ oracle calls. Then,
    \[
        \Prob\set{T\in Q} < \eta.
    \]
\end{lemma}
\begin{proof}
    Given a realization of latent variables $A_1,\dots, A_m$ generating the training sample $S$, define the set of sample-consistent target indices
    \[
        V_S := [K]\setminus \cup_{j=1}^{m} A_j.
    \]
    For any realization of $A_1,\dots, A_m$ and any $t\in V_S$,
    \[
        \Prob\set{A_1,\dots, A_m | T= t} = \prod_{j=1}^{m} p^{|A_j|}(1-p)^{K-1-|A_j|},
    \]
    which does not depend on the particular index $t$. Hence, conditioning on the random permutations $\Pi$, the training sample $S=(X_j, Y_j)_{j=1}^{m}$ and its latent variables $(A_j, R_j)_{j=1}^{m}$, $T$ is uniform on $V_S$. Therefore,
    \[
        \Prob\set{T\in Q \mid \Pi, S, (A_j, R_j)_{j=1}^{m}} = \frac{|Q \cap V_S|}{|V_S|} \leq \frac{q}{|V_S|}.
    \]
    Conditional on $T$, every other index other than $T$ is absent from all $m$ sets $A_1,\dots, A_m$ with probability $a=(1-p)^m$. Therefore $|V_S| = 1+X$ where $X \sim {\rm Bin}(K-1,a)$. Using $1/(1+k)=\int_0^1u^k\,du$,
    \[
        \E\brk{\frac1{1+X}}
        =\int_0^1(1-a+au)^{K-1}\,du
        =\frac{1-(1-a)^K}{Ka}
        \leq\frac1{Ka}.
    \]
    Taking expectations and using the fact that the choice parameters satisfies $Ka \geq d^\star/\eta, q = d^\star - 1$,
    \[
        \Prob\{T\in Q\} = \E\brk{\Prob\set{T\in Q \mid \Pi, S, (A_j, R_j)_{j=1}^{m}}} \leq q \E\brk{\frac{1}{|V_S|}} \leq\frac{q_0}{Ka} \leq\frac{q_0}{d^\star}\eta <\eta.
    \]
    This concludes the proof.
\end{proof}

\subsection{An Indistinguishable Opposite-Label Pair}
In this section, we describe a particular choice of a pair of perturbations of test examples that will have opposite labels but will be indistinguishable from the perspective of the learning algorithm conditioning on all available information during training.

Recall that there is a fixed ordering on each $Z_r$.  For every possible response $i\in([K]\cup\{\bot\})^q$ by the oracle, let $Q_i=Q(i)$ be its set of nondummy
indices.  Choose
\[
    Q_i^{+}\subseteq Q_i,
    \qquad |Q_i^{+}|=\left\lfloor\frac{|Q_i|}{2}\right\rfloor,
\]
and define
\begin{equation}\label{eq:special-types}
    \theta_i^+=(+1,Q_i^{+}),
    \qquad
    \theta_i^-=(-1,Q_i\setminus Q_i^{+}).
\end{equation}
Since $|Q_i|\leq q=2B = d^\star-1$, it follows that $|Q_i^{+}|\leq B$ and $|Q_i\setminus Q_i^{+}| =\left\lceil|Q_i|/2\right\rceil\leq B$.
Thus both $\theta_i^+, \theta_i^-$ in \eqref{eq:special-types} belong to the latent space $\Theta$.

Observe that $\theta_i^+, \theta_i^-$ are indistinguishable to all returned predictors indexed by $Q_i$. By construction of the class $\Fc$, if $j\in Q_i^{+}$, then $f_j(\theta_i^+)=f_j(\theta_i^-)=-1$. If $j\in Q_i\setminus Q_i^{+}$, then $f_j(\theta_i^+)= f_j(\theta_i^-)=+1$. Consequently,
\begin{equation}\label{eq:returned-agree}
    f_j(\theta_i^+)=f_j(\theta_i^-)
    \qquad\text{for every }j\in Q_i.
\end{equation}
In contrast, every unreturned predictor assigns opposite labels:
\begin{equation}\label{eq:unreturned-opposite}
    f_j(\theta_i^+)=+1,
    \qquad
    f_j(\theta_i^-)=-1
    \qquad\text{for every }j\notin Q_i.
\end{equation}
In particular, on the event $T\notin Q_i$, the target labels $\theta_i^+$ by $+1$ and $\theta_i^-$ by $-1$.

For a bijection $\rho:Z_r\to\Theta$, let
$z_{i,r}^+(\rho) = \rho^{-1}(\theta_i^+)$ and $z_{i,r}^-(\rho) = \rho^{-1}(\theta_i^-)$ be the instances in $Z_r$ to which $\theta_i^+$ and $\theta_i^-$ are mapped to. Write
\[
    L_{i,r}(\rho)=(u,v),\qquad u<v,
\]
for the ordered tuple of these two instances (according to the a priori fixed ordering on $Z_r$), without specifying which instance is mapped to $\theta_i^+$ and which is mapped to $\theta_i^-$. Define the orientation bit by
\begin{equation}\label{eq:omega-def}
    \Omega_{i,r}(\rho)=0
    \quad\Longleftrightarrow\quad
    z_{i,r}^+(\rho)=u,
\end{equation}
so $\Omega_{i,r}(\rho)=1$ means that $z_{i,r}^+(\rho)=v$ instead.
Let $\Phi_{i,r}(\rho) = (f_j|_{Z_r})_{j\in Q_i} \in \set{-1, 1}^{Z_r \times |Q_i|}$ be the ordered tuple, of the restrictions (or projections) of
the predictors with indices in $Q_i$ to $Z_r$, and set
\begin{equation}\label{eq:local-side-info}
    \Gamma_{i,r}(\rho)
    =\bigl(\Phi_{i,r}(\rho),L_{i,r}(\rho)\bigr).
\end{equation}
Thus $\Gamma_{i,r}$ reveals every column (predictor indexed in $Q_i$) on the block $r$ and also the
two relevant instances $u, v$ defined above, but not their orientation (which is mapped to $\theta_i^+$ and which is mapped to $\theta_i^-$).

We now verify that the above setup satisfies the requirements \eqref{eq:local-fairness} of \prettyref{lem:hidden-bit}. Let $\tau_i:\Theta\to\Theta$ exchange $\theta_i^+$ and $\theta_i^-$ and be identity otherwise.  For a block bijection $\rho$, put $\rho'=\tau_i\circ\rho$. By \eqref{eq:returned-agree}, replacing $\rho$ by $\rho'$ leaves every
returned column unchanged, $\Phi_{i, r}(\rho)=\Phi_{i, r}(\rho')$.  It also leaves the unordered location pair $L_{i,r}$ unchanged, but it flips the orientation $\Omega_{i,r}$. Hence, on every set $\set{\rho: \Gamma_{i,r}(\rho)=\gamma}$, the mapping $\rho\mapsto\rho'$ pairs each orientation-$0$ bijection with a unique orientation-$1$ bijection. Since $\Pi_r$ is uniform over all block
bijections, for every side-information value $\gamma$ of positive
probability,
\begin{equation}\label{eq:orientation-locally-fair}
    \Prob\{\Omega_{i,r}(\Pi_r)=0\mid\Gamma_{i,r}(\Pi_r)=\gamma\}
    =
    \Prob\{\Omega_{i,r}(\Pi_r)=1\mid\Gamma_{i,r}(\Pi_r)=\gamma\}
    =\frac12.
\end{equation}
Thus the orientation is uniform even after the returned columns and the
two relevant instances have been revealed.

\subsection{Low Total Variation Distance from Uniform}

The following lemma combines the information-usage method for controlling data-dependent selection bias \citep{RussoZou20} with the tensorization of relative entropy over independent coordinates \citep{MadimanTetali10}; applying data processing and Pinsker's inequality converts the resulting coordinatewise information budget into an average posterior-bias bound. Related individual-coordinate mutual-information bounds appear in \citet*{BuZouVeeravalli20}. Its proof is deferred to \prettyref{app:appendix2}.

\begin{restatable}{lemma}{itlemma}\label{lem:hidden-bit}
    Let $W_1,\ldots,W_n$ be independent random variables taking values in finite
    sets, and let $J$ be any finite-valued random variable jointly distributed
    with $W=(W_1,\ldots,W_n)$.  For each possible value $j$ of $J$ and each
    $r\in[n]$, let $\Omega_{j,r}=\omega_{j,r}(W_r)\in\{0,1\}$, $\Gamma_{j,r}=\gamma_{j,r}(W_r)$,
    and put $\Gamma_j=(\Gamma_{j,1},\ldots,\Gamma_{j,n})$.  Suppose that, under
    the prior law of $W_r$,
    \begin{equation}\label{eq:local-fairness}
        \Law(\Omega_{j,r}\mid\Gamma_{j,r})=\Unif(\{0,1\})
        \qquad\text{for every $j,r$}.
    \end{equation}
    Thus, for every fixed $j$, the bit $\Omega_{j,r}$ is fair even after the
    local side information $\Gamma_{j,r}$ is revealed.  Define the posterior
    bias
    \[
        \Delta_r=
        \TV\bigl(\Law(\Omega_{J,r}\mid J,\Gamma_J),\Unif(\{0,1\})\bigr).
    \]
    Then
    \begin{equation}\label{eq:hidden-bit}
        \E\left[\frac1n\sum_{r=1}^n\Delta_r\right]
        \leq\sqrt{\frac{H(J)}{2n}}.
    \end{equation}
\end{restatable}

In our context, we will condition on the following information. Fix the random target $T$, the learner's randomness $L$, the training sample $S=((X_j, Y_j))_{j=1}^{m}$ and its latent variables $((R_j, A_j))_{j=1}^{m}$, the permutations of all blocks appearing in the training sample $(\Pi_{R_j})_{j=1}^{m}$. Define the set of unseen blocks in the domain by $U=[N] \setminus \set{R_1,\dots, R_m}$, where $|U|:= n \geq N - m$.

We will condition on $G=(T, L, S, (R_j, A_j, \Pi_{R_j})_{j=1}^{m})$. The random variables $(\Pi_r)_{r\in U}$ remain independent uniform bijections.  Relabeling the
$n$ (unseen) coordinates if necessary, we apply \prettyref{lem:hidden-bit} with
\[
    W_r=\Pi_r,
    \qquad J=I,
    \qquad
    \Omega_{i,r}=\Omega_{i,r}(\Pi_r),
    \qquad
    \Gamma_{i,r}=\Gamma_{i,r}(\Pi_r).
\]
The uniformity requirement of the orientation bit $\Omega_{i,r}$ \eqref{eq:local-fairness} is satisfied by \eqref{eq:orientation-locally-fair}. The oracle response $I$ takes values in
$([K]\cup\{\bot\})^q$, so for every realization $g$ of $G$,
\begin{equation}\label{eq:transcript-entropy}
    H(I\mid G=g)\leq q\log(K+1).
\end{equation}
Let $\Gamma_I=(\Gamma_{I,u}(\Pi_u))_{u\in\mathsf U}$ and, for $r\in\mathsf U$, define
\begin{equation}\label{eq:orientation-bias-definition}
    \Delta_r=
    \TV\bigl(
    \Law(\Omega_{I,r}(\Pi_r)\mid G,I,\Gamma_I),
    \Unif(\{0,1\})
    \bigr).
\end{equation}
\prettyref{lem:hidden-bit}, \eqref{eq:transcript-entropy}, and
$n\geq N-m$ imply, for every realization $g$ of $G$,
\[
    \E\left[
    \frac1n\sum_{r\in\mathsf U}\Delta_r
    \;\middle|\; G=g
    \right]
    \leq
    \sqrt{\frac{q\log(K+1)}{2n}}
    \leq
    \sqrt{\frac{q\log(K+1)}{2(N-m)}}
    <\eta.
\]
Averaging over $G$ gives
\begin{equation}\label{eq:orientation}
    \E\left[\frac1n\sum_{r\in U}\Delta_r\right]<\eta.
\end{equation}
Let $R$ be the block index of an independent test example, and put
$\Delta_R=0$ when $R\notin U$.  The test index $R$ is uniform on $[N]$ and
independent of the hidden block bijections. Together with \eqref{eq:orientation}, this yields
\begin{equation}\label{eq:test-orientation}
    \E[\1\{R\in U\}\Delta_R]= \frac{n}{N}\E\brk{\frac1n\sum_{r\in U}\Delta_r} <\eta.
\end{equation}

\subsection{Putting the Pieces Together}

\begin{proof}[Proof of \prettyref{thm:lowerbound}]

    Draw an independent test example $(X,Y)$ with latent variables $(R,A)$ and define the event
    \begin{equation}\label{eq:good-event}
        \mathcal E
        =\{R\in U,\ T\notin Q(I),\ Q(I)\subseteq A\}.
    \end{equation}
    On event $\mathcal E$, both $Q^+_I$ and $Q(I)\setminus Q^+_I$ are subsets of $A$.
    Therefore
    \[
        z_{I,R}^+(\Pi_R)\in V_\Pi(R,A,+1),
        \qquad
        z_{I,R}^-(\Pi_R)\in V_\Pi(R,A,-1).
    \]
    Define the witness instance
    \[
        Z^{\mathrm{wit}}
        =\begin{cases}
            z_{I,R}^+(\Pi_R),&\text{if }Y=+1,\\
            z_{I,R}^-(\Pi_R),&\text{if }Y=-1.
        \end{cases}
    \]
    Then, on event $\mathcal E$, the robust loss is bounded from below as follows
    \begin{equation}\label{eq:witness-lower}
        \sup_{z\in \Uc(X)} \1\set{H(z)\neq Y} = \sup_{z\in V_\Pi(R,A,Y)} \1\set{H(z)\neq Y} \geq\1\{H(Z^{\mathrm{wit}})\neq Y\}.
    \end{equation}

    Condition on $G,I,\Gamma_I,R,A$ and suppose event $\mathcal E$ holds.
    Recall that $L_{I,R}(\Pi_R)=(u,v), u<v$ (either $u=z_{I,R}^+(\Pi_R)$ and $v=z_{I,R}^-(\Pi_R)$, or vice-versa).
    The classifier values $H(u),H(v)$ are fixed under this conditioning.  Under
    our convention for $\Omega_{I,R}$ \eqref{eq:omega-def}, the four possibilities are summarized in the following table
    \[
        \begin{array}{c|cc}
            &Y=+1&Y=-1\\ \hline
            \Omega_{I,R}=0 & H(u) \overset{?}{=} +1 & H(v) \overset{?}{=} -1\\
            \Omega_{I,R}=1 & H(v) \overset{?}{=} +1 & H(u) \overset{?}{=} -1.
        \end{array}
    \]
    If the $\Omega_{I,R}$ and $Y$ are both uniform, the point $u$ is tested once
    against $+1$ and once against $-1$, so exactly one of those two tests is an
    error. The same is true for $v$. Hence, for every fixed pair
    $H(u),H(v)$, exactly two of the four equally likely cases produce an error. Therefore, it follows that
    \begin{equation}\label{eq:fair-orientation-half}
        \Ex{\Omega\sim\Unif(\{0,1\})}\brk{g(\Omega)} = \frac{1}{2},\qquad g(\omega)
        =\Prob_Y\{H(Z^{\mathrm{wit}})\neq Y
            \mid\Omega_{I,R}=\omega,
            G,I,\Gamma_I,R,A\}.
    \end{equation}

    The actual posterior law of $\Omega_{I,R}$ may not be uniform, but we can relate it to uniform by the variational characterization of total variation,
    \begin{equation}\label{eq:posterior-error}
        \E[g(\Omega_{I,R})]\geq \Ex{\Omega\sim\Unif(\{0,1\})}\brk{g(\Omega)} - \Delta_R \geq\frac12-\Delta_R,
    \end{equation}
    where $\Delta_R$ is the total variation distance \eqref{eq:orientation-bias-definition}.

    Combining \eqref{eq:witness-lower} and \eqref{eq:posterior-error}, then
    averaging, and then \eqref{eq:test-orientation}, gives
    \begin{align}
        \E\Risk(H;P_{\Pi,T})
        &\geq
        \E\brk{\1_{\mathcal E} \E\brk{\1\set{H(Z^{\mathrm{wit}})\neq Y} \middle| G,I,\Gamma_I, R, A}}\notag\\
        &\geq
        \E\left[
        \1_{\mathcal E}\left(\frac12-\Delta_R\right)
        \right] \notag\\
        &\geq
        \frac12\Prob(\mathcal E)
        -\E[\1\{R\in\mathsf U\}\Delta_R]\notag\\
        &\geq
        \frac12\Prob(\mathcal E)-\eta.                     \label{eq:risk-lower}
    \end{align}

    \paragraph{Probability of the good event.} It remains to bound the probability of event $\mathcal E$ \eqref{eq:good-event}, $\Prob(\mathcal E)$, from below. Conditioning on $G,I,\Gamma_I$ fixes $T, Q(I), U$, but $A$ remains random. If $T\in Q$, then the event $Q\subseteq A$ is impossible because $A \subseteq [K]\setminus {T}$ (as specified in $P_{\Pi, T}$). If $T\notin Q$, then every index in $Q$ is included in $A$ with probability $p$. Thus
    \[
        \Prob\set{Q\subseteq A \mid G, I, \Gamma_I} = \1\set{T\notin Q}p^{|Q|}\geq \1\set{T\notin Q}p^{q}.
    \]
    Consequently,
    \[
        \Prob(\mathcal E) = \E\brk{\1\set{R\in U, T\notin Q}p^{|Q|}} \geq p^{q} \Prob\set{R\in U, T\notin Q}.
    \]
    By a union bound,
    \[
        \Prob\set{R\in U, T\notin Q} \geq 1 - \Prob\set{R\notin U} - \Prob\set{T\in Q}.
    \]
    Conditioning on the training blocks $R_1,\dots, R_m$, $\Prob\set{R\notin U \mid R_1,\dots, R_m}\leq m/N$. Therefore,
    
    $\Prob\set{R\notin U}\leq m/N$. \prettyref{lem:target-return} shows that $\Prob\set{T\notin Q} < \eta$. Thus,
    \begin{equation} \label{eqn:eventlowerbound}
        \Prob(\mathcal E) \geq p^q \prn{1-\frac{m}{N} - \eta}.
    \end{equation}
    Note that the choice of parameters in the construction ensures that $m/N < \eta$. By definition, recall that $p = 1 - \eta/d^\star$ and $q=d^\star - 1$. By Bernoulli's inequality, $p^q=(1-\eta/d^\star)^{d^\star-1}\geq 1- (d^\star-1)\eta/d^\star > 1-\eta$. Combining the above, with the choice of $\eta=1/1000$, it follows that
    \[
        \E_{\Pi, T}\brk{\Risk(H; P_{\Pi, T})} > \frac{1}{2} (1-\eta)(1-2\eta) - \eta > 0.49.
    \]
    By the probabilistic method, there must exist a fixed pair $\pi, t$ (which may depend on learner $\mathbb{A}$) such that $\E\brk{\Risk(H; P_{\pi, t})} > 0.49$. Applying a variant of Markov's inequality concludes the proof.
\end{proof}

\section*{AI Disclosure}

The author used ChatGPT Pro (5.5 and 5.6) and Codex during the development of several ideas and proofs presented in this paper. The author also used these models to assist with some of the writing.

\bibliographystyle{plainnat}
\bibliography{refs}

\appendix

\section{Auxiliary Lemmata for Upper Bound}
\label{app:appendix1}

The proof is a straightforward extension of the suffix averaging argument of \citet*[][Theorem 2.1]{DBLP:conf/focs/Aden-AliCSZ23} applied to our margin-robust loss. Observe that the bagging predictor \eqref{eqn:full-bagging} is \emph{permutation invariant}. That is, for any sequence $S$ and permutation $\pi$, $\widehat{\B}_S = \widehat{\B}_{S_\pi}$. Furthermore, by \prettyref{lem:leave-one-out-margin}, its leave-one-out error with respect to the margin-robust loss $(x,y) \mapsto \sup_{z\in \Uc(x)} \1\{y\widehat{\B}_S(z)\leq \gamma\}$ is bounded from above by $O(d/n(1-\gamma)^2)$.

The following Lemmas are the margin-robust counterparts of the forward and
reverse martingale bounds in Lemmas~4.2 and~4.3
\citep*{DBLP:conf/focs/Aden-AliCSZ23}. Their proofs apply verbatim:
the displayed margin-robust loss is measurable and takes values in
\([0,1]\), while the reverse martingale argument uses only
bounded leave-one-out error (\prettyref{lem:leave-one-out-margin}) and
permutation invariance.

\begin{lemma}[Forward martingale bound, Lemma 4.2 in \cite*{DBLP:conf/focs/Aden-AliCSZ23}]
    \label{lem:margin-robust-forward}
    Let $S=((X_1,Y_1),\ldots,(X_n,Y_n))\sim P^n$. For every \(\eta,\rho\in(0,1)\), with probability at least \(1-\rho\),
    \[
        \begin{split}
            &\sum_{t=n/4}^{n-1}
            \Ex{(X,Y)\sim P}
            \brk{
                \sup_{Z\in\Uc(X)}
                \1\left\{
                    Y\widehat{\B}_{S_{\leq t}}(Z)\leq\gamma
                    \right\}
            }
            \\
            &\qquad\leq
            \frac{\eta e^\eta}{e^\eta-1}
            \sum_{t=n/4}^{n-1}
            \sup_{Z\in\Uc(X_{t+1})}
            \1\left\{
                Y_{t+1}\widehat{\B}_{S_{\leq t}}(Z)\leq\gamma
                \right\}
            +
            \frac{e^\eta}{e^\eta-1}\log\frac{1}{\rho}.
        \end{split}
    \]
\end{lemma}

\begin{lemma}[Reverse martingale bound, Lemma 4.3 in \cite*{DBLP:conf/focs/Aden-AliCSZ23}]
    \label{lem:margin-robust-reverse}
    Suppose that $P$ is a distribution over $\X\times \Y$ satisfying $\inf_{f^\star \in \Fc} \Risk(f^\star; P)=0$ (i.e., robustly realizable). Let $S=((X_1,Y_1),\ldots,(X_n,Y_n))\sim P^n$. Let $M_n=O(d/(1-\gamma)^2)$. For every \(\lambda,\rho\in(0,1)\), with probability at least
    \(1-\rho\),

    \[
        \sum_{t=n/4}^{n-1} \sup_{Z\in\Uc(X_{t+1})}\1\left\{Y_{t+1}\widehat{\B}_{S_{\leq t}}(Z)\leq\gamma \right\} \leq \frac{e^\lambda-1}{\lambda}
        M_n \sum_{m=n/4+1}^{n}\frac{1}{m} + \frac{1}{\lambda}\log\frac{1}{\rho}.
    \]
\end{lemma}

\begin{proof}[Proof of \prettyref{lem:high-prop}]
    Apply \prettyref{lem:margin-robust-forward}
    and \prettyref{lem:margin-robust-reverse} with confidence parameter
    \(\rho=\delta/2\). By a union bound, with probability at least
    \(1-\delta\),
    \[
        \begin{split}
            \sum_{t=n/4}^{n-1}
            \Ex{(X,Y)\sim P}
            &\brk{
                \sup_{Z\in\Uc(X)}
                \1\left\{
                    Y\widehat{\B}_{S_{\leq t}}(Z)\leq\gamma
                    \right\}
            }
            \\
            &\quad\leq
            \frac{\eta e^\eta}{e^\eta-1}
            \left(
            \frac{e^\lambda-1}{\lambda}
            M_n
            \sum_{m=n/4+1}^{n}\frac{1}{m}
            +
            \frac{1}{\lambda}\log\frac{2}{\delta}
            \right)
            +
            \frac{e^\eta}{e^\eta-1}\log\frac{2}{\delta}.
        \end{split}
    \]
    Since \(x\mapsto 1/x\) is decreasing,
    \[
        \sum_{m=n/4+1}^{n}\frac{1}{m}
        \leq
        \int_{n/4}^{n}\frac{dx}{x}
        =
        \log 4.
    \]
    Consequently,

    \[
        \begin{aligned}
            &\frac{4}{3n} \sum_{t=n/4}^{n-1} \Ex{(X,Y)\sim P} \brk{\sup_{Z\in\Uc(X)}\1\left\{Y\widehat{\B}_{S_{\leq t}}(Z)\leq\gamma\right\}} \\
            &\qquad\leq
            \frac{4}{3}
            \frac{\eta e^\eta}{e^\eta-1}
            \log(4)
            \frac{e^\lambda-1}{\lambda}
            \frac{M_n}{n}
            +
            \frac{4}{3}
            \left(
            \frac{\eta e^\eta}
            {\lambda(e^\eta-1)}
            +
            \frac{e^\eta}{e^\eta-1}
            \right)
            \frac{1}{n}\log\frac{2}{\delta}.
        \end{aligned}
    \]
    Taking $\lambda=0.82$ and $\eta=0.78$ gives
    \[
        \frac{4}{3}
        \frac{\eta e^\eta}{e^\eta-1}
        \log(4)
        \frac{e^\lambda-1}{\lambda}
        \approx 4.1246<4.82
        \quad
        \text{and}
        \quad
        \frac{4}{3}
        \left(
        \frac{\eta e^\eta}
        {\lambda(e^\eta-1)}
        +
        \frac{e^\eta}{e^\eta-1}
        \right)
        \approx 4.8037<4.82.
    \]
    It follows that
    \[
        \frac{4}{3n}
        \sum_{t=n/4}^{n-1}
        \Ex{(X,Y)\sim P}
        \brk{
            \sup_{Z\in\Uc(X)}
            \1\left\{
                Y\widehat{\B}_{S_{\leq t}}(Z)\leq\gamma
                \right\}
        }
        \leq
        4.82
        \prn{
            \frac{M_n}{n}
            +
            \frac{1}{n}\log\frac{2}{\delta}
        },
    \]
    which proves the lemma.
\end{proof}

\section{Auxiliary Lemmata for Lower Bound}
\label{app:appendix2}

\itlemma*

\begin{proof}
    The point of the lemma is that the transcript value $J=j$ may select a
    different bit $\Omega_{j,r}$ and different side information $\Gamma_{j,r}$.
    We therefore keep $j$ fixed until the last averaging step.

    Fix $j$ and write $\Gamma_j=(\Gamma_{j,1},\ldots,\Gamma_{j,n})$.  Since
    $\Gamma_j$ is a deterministic function of $W$, the relative-entropy chain
    rule gives
    \begin{align}
        \KL\bigl(\Law(W\mid J=j)\,\|\,\Law(W)\bigr)
        &=
        \KL\bigl(\Law(\Gamma_j\mid J=j)\,\|\,\Law(\Gamma_j)\bigr) \notag\\
        &\qquad+
        \E\left[
        \KL\bigl(
        \Law(W\mid J=j,\Gamma_j)
        \,\|\,
        \Law(W\mid\Gamma_j)
        \bigr)
        \;\middle|\;J=j
        \right].                                                     \label{eq:KL-chain}
    \end{align}
    The first term on the right is nonnegative.  Multiplying by
    $\Prob\{J=j\}$, summing over $j$, and using the definition of mutual
    information therefore yields
    \begin{equation}\label{eq:conditional-KL-budget}
        \sum_j\Prob\{J=j\}
        \E\left[
        \KL\bigl(
        \Law(W\mid J=j,\Gamma_j)
        \,\|\,
        \Law(W\mid\Gamma_j)
        \bigr)
        \;\middle|\;J=j
        \right]
        \leq I(W;J)\leq H(J).
    \end{equation}
    The second law in each summand conditions on the fixed statistic
    $\Gamma_j$ but not on the event $J=j$; the KL divergence therefore measures
    the additional information supplied by that event.

    We next decompose the left side of \eqref{eq:conditional-KL-budget} across
    coordinates.  For fixed $j$ and every $\gamma$ in the support of $\Gamma_j$,
    independence of the $W_r$ and the coordinatewise form of $\Gamma_j$ imply
    \begin{equation}\label{eq:conditional-product}
        \Law(W\mid\Gamma_j=\gamma)
        =\bigotimes_{r=1}^n
        \Law(W_r\mid\Gamma_{j,r}=\gamma_r).
    \end{equation}
    Indeed, before normalization the conditional mass at
    $w=(w_1,\ldots,w_n)$ factors as
    \[
        \prod_{r=1}^n
        \Prob\{W_r=w_r\}
        \1\{\gamma_{j,r}(w_r)=\gamma_r\}.
    \]
    For an arbitrary law $P$ on a product space, with marginals $P_r$, and a
    product reference law $\bigotimes_r Q_r$, one has
    \begin{equation}\label{eq:KL-marginals}
        \KL\left(P\,\middle\|\,\bigotimes_{r=1}^nQ_r\right)
        =\KL\left(P\,\middle\|\,\bigotimes_{r=1}^nP_r\right)
        +\sum_{r=1}^n\KL(P_r\|Q_r)
        \geq\sum_{r=1}^n\KL(P_r\|Q_r).
    \end{equation}
    Applying \eqref{eq:KL-marginals} conditionally, with the product reference
    law in \eqref{eq:conditional-product}, and then averaging gives
    \begin{align}
        &\sum_j\Prob\{J=j\}
        \E\left[
        \sum_{r=1}^n
        \KL\bigl(
        \Law(W_r\mid J=j,\Gamma_j)
        \,\|\,
        \Law(W_r\mid\Gamma_{j,r})
        \bigr)
        \;\middle|\;J=j
        \right]
        \leq H(J).                                      \label{eq:coordinate-KL-budget}
    \end{align}
    Again, the second law in each summand conditions only on the fixed local
    statistic $\Gamma_{j,r}$, not on the event $J=j$.

    For each fixed $j,r$, the map $W_r\mapsto\Omega_{j,r}$ is deterministic.
    Applying the data processing inequality to every summand in
    \eqref{eq:coordinate-KL-budget}, followed by
    \eqref{eq:local-fairness}, gives
    \begin{align}
        &\E\brk{\sum_{r=1}^{n} \KL\bigl(\Law(\Omega_{J,r}\mid J,\Gamma_J)\,\|\,\Unif(\{0,1\})\bigr)} \notag \\
        &= \sum_{j} \Prob\set{J=j} \E\brk{\sum_{r=1}^{n} \KL\bigl(\Law(\Omega_{j,r}\mid J=j,\Gamma_j)\,\|\,\Law(\Omega_{j,r}\mid\Gamma_{j,r})\bigr) \;\middle|\;J=j } \notag \\
        &\leq \sum_j\Prob\{J=j\}\E\left[\sum_{r=1}^n\KL\bigl(\Law(W_r\mid J=j,\Gamma_j)\,\|\,
        \Law(W_r\mid\Gamma_{j,r})\bigr)\;\middle|\;J=j\right] \notag\\
        &\leq H(J). \label{eq:bit-KL-budget}
    \end{align}
    Applying Pinsker's inequality, with natural logarithms, gives
    $\Delta_r^2\leq D_r/2$, where $D_r= \KL\bigl(\Law(\Omega_{J,r}\mid J,\Gamma_J)\,\|\,\Unif(\{0,1\})\bigr)$ and $\Delta_r$ is the total variation distance. Finally, by Cauchy--Schwarz and Jensen's inequality,
    \[
        \E\brk{\frac1n\sum_{r=1}^n\Delta_r}
        \leq
        \sqrt{\E\brk{\frac1n\sum_{r=1}^n\Delta_r^2}} \leq \sqrt{\E\brk{\frac{1}{2n}\sum_{r=1}^{n}D_r}}\leq \sqrt{\frac{H(J)}{2n}}.
    \]
    This proves \eqref{eq:hidden-bit}.
\end{proof}

\end{document}